\documentclass[a4paper,fleqn]{cas-dc}
\usepackage[numbers,sort&compress]{natbib}

\def\tsc#1{\csdef{#1}{\textsc{\lowercase{#1}}\xspace}}
\tsc{WGM}
\tsc{QE}
\tsc{EP}
\tsc{PMS}
\tsc{BEC}
\tsc{DE}
\usepackage{amsmath,amssymb,amsthm,booktabs}
\usepackage{makecell}
\usepackage{algorithm,algpseudocode}
\usepackage{capt-of}
\usepackage{placeins}
\usepackage{needspace}
\renewcommand{\algorithmicrequire}{\textbf{Input:}}
\renewcommand{\algorithmicensure}{\textbf{Output:}}
\usepackage{graphicx}
\usepackage{xcolor}
\hypersetup{hypertexnames=false}
\usepackage{multirow}
\usepackage{pifont}
\newcommand{\cmark}{\ding{51}}
\newcommand{\xmark}{\ding{55}}

\newtheorem{proposition}{Proposition}
\newtheorem{lemma}{Lemma}
\newtheorem{corollary}{Corollary}
\newtheorem{assumption}{Assumption}
\newtheorem{definition}{Definition}

\usepackage{dsfont}
\newcommand{\ind}{\mathds{1}} 
\usepackage{xspace}
\newcommand{\methodname}{ProME\xspace}
\newcommand{\rev}[1]{#1} 
\definecolor{consistencyPurple}{RGB}{128,0,160}
\renewcommand{\paragraph}[1]{\par\smallskip\noindent\textbf{\upshape #1}\enspace\ignorespaces}
\hypersetup{colorlinks=true,linkcolor=black,citecolor=black,urlcolor=black}
\renewcommand{\figurename}{Fig.}

\begin{document}

\shorttitle{ProME: Group Robustness without Training Group Labels}
\shortauthors{Q. Wang et al.}

\title[mode=title]{ProME: Prototype-Margin Environments with Repair-Aware Selection for Group-Robust Learning}

\author[1,2]{Qianqian Wang}
\ead{12331103@mail.sustech.edu.cn}

\author[2]{Yunshan Li}
\ead{12331105@mail.sustech.edu.cn}

\author[2]{Dawei Huang}
\ead{12312405@mail.sustech.edu.cn}

\author[2]{Wenwu Gong}
\cormark[1]
\ead{gongww@sustech.edu.cn}

\author[1,2]{Lili Yang}
\cormark[1]
\ead{yangll@sustech.edu.cn}

\affiliation[1]{organization={Shenzhen Key Laboratory of Safety and Security for Next Generation of Industrial Internet},
            city={Shenzhen},
            country={China}}
\affiliation[2]{organization={Department of Statistics and Data Science, Southern University of Science and Technology},
            city={Shenzhen},
            country={China}}
\cortext[cor1]{Corresponding authors}
\newcommand{\std}[1]{{\scriptsize$\pm$#1}}

\begin{abstract}
Group-robust learning is crucial for maintaining accuracy on rare subpopulations when training-group labels are unavailable. However, existing methods often infer environments from a separate reference model and select representations before fitting the classifier used at deployment, leaving both decisions misaligned with the deployed predictor. In this work, we formulate group robustness without training-group labels as the \emph{endogenous environments with repair-aware selection} (ERAS) problem, and propose \textbf{ProME} (\textbf{Pro}totype-\textbf{M}argin \textbf{E}nvironments) to align both decisions with the deployed predictor. ProME splits prototype margins at their median to construct approximately balanced environments along the training trajectory, and fits a group-balanced linear head on group-annotated validation data to rank the resulting predictors by validation worst-group accuracy.  We theoretically bound the worst risk across the inferred environments for a fixed predictor and partition, showing that this bound transfers to the oracle groups under an explicit alignment condition. Extensive experiments show that prototype margins enrich shortcut-conflicting examples, classifier repair reshapes candidate evaluation, and ProME achieves the highest average worst-group accuracy among the compared methods with the same group-label access.
\end{abstract}


\maketitle


\section{Introduction}
\label{sec:intro}

Group-robust learning seeks to preserve accuracy across label--attribute groups whose proportions differ between training and deployment, as in medical diagnosis, content moderation, and attribute recognition~\citep{sagawa2019distributionally,yang2023change}. Rare groups contribute little to the average loss, allowing empirical risk minimization (ERM) to achieve high overall accuracy while accuracy on rare groups remains low. This disparity often arises when ERM relies on shortcut features whose correlations with labels hold for common groups but fail for rare groups~\citep{ye2026the}. For example, background and gender can serve as shortcut features for predicting bird type and hair color in Waterbirds and CelebA, respectively~\citep{sagawa2019distributionally,geirhos2020shortcut,liu2015deep}. Errors caused by these shortcuts have greater impact when rare groups become more prevalent at deployment. Worst-group accuracy (WGA) captures this vulnerability by reporting the accuracy of the least accurate group instead of averaging across groups~\citep{sagawa2019distributionally}. Fig.~\ref{fig:overview} illustrates this setting through the Waterbirds training-group proportions and reported WGA on three benchmarks. The challenge is to maintain WGA as group proportions shift.

Group distributionally robust optimization (Group DRO) targets WGA by minimizing the worst empirical group risk~\citep{sagawa2019distributionally,liu2021just,pezeshki2024discovering}. It requires a group label for every training example, yet such labels are costly or infeasible to collect for large training datasets. We use \emph{oracle groups} to denote the label--attribute groups whose labels are unavailable during training. Methods without training-group labels commonly use a two-stage pipeline. A reference model infers environments by clustering learned features~\citep{sohoni2020no,creager2021environment,han2024improving} or identifying misclassified examples~\citep{liu2021just,zhang2022correct,qiu2023simple,pezeshki2024discovering}. A separate predictor is trained using the resulting fixed assignments~\citep{sohoni2020no,creager2021environment,liu2021just}. This separation creates an environment--representation mismatch. This mismatch can weaken invariant learning~\citep{liao2026decorr,wang2026environment} and limit WGA on the oracle groups.

\noindent\begin{minipage}{\columnwidth}
  \centering
  \includegraphics[width=\columnwidth]{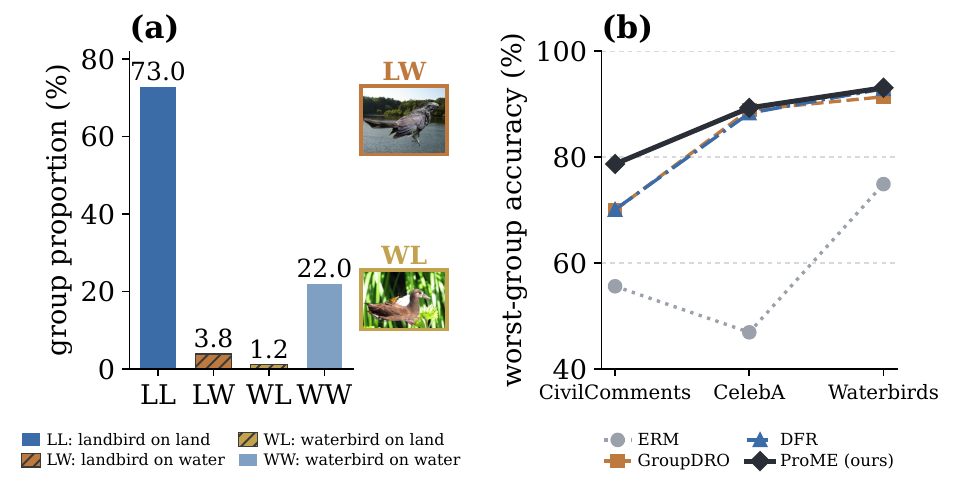}
  \refstepcounter{figure}\label{fig:overview}\par\smallskip
  {\small\textbf{\figurename~\thefigure.} \textbf{Training-group imbalance and worst-group accuracy.} \textbf{(a)}~Group proportions in the Waterbirds training set. \textbf{(b)}~Reported WGA of representative methods, including deep feature reweighting (DFR), on three benchmarks.}
\end{minipage}

A second misalignment arises during model selection, when an encoder checkpoint must be chosen for deployment. An encoder checkpoint with low pre-repair WGA can still retain features that support a stronger classifier after repair. Classifier repair refits the last layer on group-balanced data~\citep{kirichenko2023last,labonte2024towards,hill2026on}. Selecting encoder checkpoints by pre-repair WGA, or retaining the final checkpoint by default, can reject the checkpoint with the highest post-repair WGA. Together, the environment--representation mismatch and pre-repair model selection leave both decisions misaligned with the deployed predictor. This limitation raises a pivotal issue:
\begin{center}
\begin{minipage}{0.94\linewidth}
\emph{How can we align environment construction and model selection with the deployed predictor without training-group labels?}
\end{minipage}
\end{center}

In this work, we formulate group robustness without training-group labels as the \emph{endogenous environments with repair-aware selection} (ERAS) problem to align both decisions with the deployed predictor (Definition~\ref{def:endogenous}), and propose \textbf{ProME} (\textbf{Pro}totype-\textbf{M}argin
\textbf{E}nvironments) to implement this formulation. In ERAS, an environment partition is \emph{endogenous} when an encoder checkpoint creates the partition used to regularize later checkpoints on the same trajectory. Model selection is repair-aware when encoder checkpoints are ranked by WGA after classifier repair. To construct endogenous environments, ProME assigns every example a \emph{prototype margin} between its observed class and the strongest competing class under a cosine-prototype classifier. A median split of these margins produces two approximately balanced environments for invariant learning. In Stage~2, ProME freezes each retained encoder checkpoint, fits a group-balanced linear head on group-annotated validation data, and deploys the encoder--head pair with the highest validation WGA (Eq.~\eqref{eq:candidate-selection}). Environment construction and representation learning use only inputs and class labels. Group-annotated validation data support candidate tracking, classifier repair, regularization tuning, and selection; they do not update the Stage~1 representation. Overall, ProME constructs environments without training-group labels, ranks candidates after classifier repair, and deploys a single encoder--head pair.

Theoretically, we characterize three Stage~1 properties under explicit conditions. First, low prototype margins mark shortcut-conflicting examples when the causal and spurious components are sufficiently distinguishable (Proposition~\ref{prop:Tmargin}). Second, the induced assignments change by a controlled amount as the representation changes (Proposition~\ref{prop:T0}). Third, for a fixed predictor and partition, we bound the worst inferred-environment risk by the mean and variance of the environment risks (Proposition~\ref{prop:T2}). Under an explicit total-variation alignment condition, this bound transfers to the oracle groups (Corollary~\ref{cor:T2-oracle}). Classifier repair and repair-aware selection in Stage~2 are evaluated experimentally.

To evaluate ProME and examine its two alignment mechanisms, we study three real-world benchmarks (Waterbirds~\citep{sagawa2019distributionally}, CelebA~\citep{liu2015deep}, and CivilComments~\citep{koh2021wilds}) and one controlled benchmark (ColoredMNIST~\citep{arjovsky2019invariant}). The main comparison uses published baseline results under their original protocols. Controlled studies use matched settings. They show that trajectory-derived margins improve pre-repair WGA in the matched source comparison and that classifier repair reshapes candidate evaluation. Among the methods compared with the same group-label access, ProME raises average WGA from 83.9\% to 87.0\%.

Our contributions are summarized as follows:
\begin{itemize}
\item \textbf{Deployment-aligned formulation.}
We formulate group robustness without training-group labels as endogenous environments with repair-aware selection (ERAS). ERAS aligns environment construction and model selection with the deployed predictor.
\item \textbf{Two-stage ProME framework.}
We propose \textbf{ProME}, a two-stage framework that constructs approximately balanced environments from prototype margins and ranks encoder--head pairs after classifier repair. It requires no training-group labels, is backbone-agnostic, and deploys only the selected encoder--head pair.
\item \textbf{Mechanism-aligned theory.}
We establish theoretical results. Under explicit conditions, low prototype margins mark shortcut-conflicting examples, and the induced partition changes by a controlled amount with the representation. We also bound the worst inferred-environment risk for a fixed predictor and partition. This bound transfers to oracle groups under an explicit alignment condition.

\end{itemize}
\section{Related Work}\label{sec:related}
ProME builds on three lines of work: environment inference without training-group labels, prototype geometry for environment inference, and classifier repair with group-annotated validation data.

\paragraph{Environment inference without training-group labels.}
When group labels are available, Group DRO directly optimizes the worst empirical group risk, and group-balanced training provides a strong supervised baseline~\citep{sagawa2019distributionally,idrissi2022simple}. Without training-group labels, error-based methods identify hard or misclassified examples with a reference classifier and use them for reweighting or contrastive learning~\citep{nam2020learning,liu2021just,zhang2022correct,li2024bias}. Other methods exploit feature clusters, confidence, disagreement, or pretrained representations to identify spurious structure or estimate pseudo-groups~\citep{sohoni2020no,zheng2024selfguided,han2024disagreement,han2024improving,tsirigotis2024group,le2025out}. Counterfactual alignment provides a complementary diagnostic for spurious correlations in black-box classifiers~\citep{cohen2025identifying}. Environment-inference methods instead optimize a partition through invariant-learning or cross-risk signals~\citep{creager2021environment,pezeshki2024discovering,chen2024neuralcollapse,stromberg2024robustness}. In these pipelines, an auxiliary model, reference predictor, or clustering stage produces environments that are subsequently used by a robust objective.
This separation can create an environment--representation mismatch when the partition does not track the representation being optimized. ProME derives its inferred environments from the representation-learning trajectory itself.

\paragraph{prototype geometry for environment inference.}
Invariant learning seeks predictors that remain stable across environments~\citep{peters2016causal,rojas2018invariant,arjovsky2019invariant,ahuja2020invariant,liu2021kernelized}. Learning such predictors without a supplied environment partition requires additional information or an explicit inductive bias~\citep{lin2022zin}, and weak variation in the relevant shortcut can leave invariant objectives with a non-robust solution~\citep{rosenfeld2021risks,kamath2021does}. Prototypical classifiers summarize each class by a mean representation, while neural collapse relates class means to terminal classifier geometry~\citep{snell2017prototypical,papyan2020prevalence}. Diverse prototypical ensembles further use multiple class prototypes to capture latent subpopulations~\citep{to2025diverse}. ProME gives prototype geometry a different role: the observed-class prototype margin from the current representation-learning trajectory defines the environments used for invariant learning.

\paragraph{Classifier repair and repair-aware selection.}
ERM representations can retain core features even when the corresponding classifier has poor worst-group accuracy~\citep{izmailov2022feature,kirichenko2023last}. Deep feature reweighting (DFR) freezes the encoder and retrains its last layer on group-balanced held-out data~\citep{kirichenko2023last}. Subsequent work reduces the annotation requirements of last-layer retraining or revisits its implementation choices~\citep{labonte2024towards,nguyen2024group,stromberg2024for}. Automatic feature reweighting (AFR), group-aware priors, and related two-stage corrections also improve group robustness through a repaired or reweighted final layer~\citep{qiu2023simple,rudner2024mind,stromberg2024enhancing}. These studies establish the value of classifier repair, yet the last-layer methods above typically evaluate a fixed or preselected representation. Their setup leaves model selection tied to a fixed or preselected representation. ProME fits the same type of group-balanced linear head to every retained encoder and ranks the resulting encoder--head pairs by validation WGA. This repair-aware comparison aligns model selection with the deployed predictor. Stage~1 environment construction and representation learning use no training-group labels; group-annotated validation data support candidate tracking, classifier repair, regularization tuning, and selection without updating the Stage~1 representation.

\section{Preliminaries}\label{sec:prelim}
In this section, we introduce subpopulation shift and group-robust risk, review invariant learning when environments are observed, and define endogenous environment inference and repair-aware selection when they are not. For reference, The notations used in this paper are presented in Table~\ref{tab:notation}.

\begin{table*}[t]
\centering
\caption{\textbf{Notations.}}
\label{tab:notation}
\small
\setlength{\tabcolsep}{3pt}
\begin{tabular}{@{}p{0.08\textwidth}p{0.40\textwidth}||p{0.08\textwidth}p{0.40\textwidth}@{}}
\toprule
Notation & Explanation & Notation & Explanation \\
\midrule
$x\in\mathcal X$ & input
& $y\in\{0,1\}$ & binary target label \\
$a$ & spurious attribute used to define oracle groups
& $g\in\mathcal G$ & oracle group; $\mathcal G$ is the set of oracle groups \\
$\mathcal D_{\mathrm{tr}},\mathcal D_{\mathrm{val}}$ & training data without group labels and validation data with group labels
& $Q_g,\pi_g^{\mathrm{tr}}$ & distribution of oracle group $g$ and its training proportion \\
$b_\xi,p_\omega$ & backbone and projection head
& $\phi_{\xi,\omega}$ & normalized representation map defined by $b_\xi$ and $p_\omega$ \\
$\phi_{\mathrm{ref}}$ & fixed auxiliary encoder used by exogenous environment-inference methods
& $s(x;y)$ & observed-class prototype similarity minus the strongest competing similarity, Eq.~\eqref{eq:proto-margin-score}
\\
$\tau$ & temperature of the prototype classifier
& $f_\phi(x)$ & prototype logit, Eq.~\eqref{eq:proto-logit} \\
$\mu_c$ & normalized prototype of class $c$ 
& $m$ & median prototype margin on $\mathcal D_{\mathrm{tr}}$ \\
$\Pi(\phi)$ & split into low- and high-margin cells at $m$, Eq.~\eqref{eq:proto-margin-split}
& $\hat{\mathcal E}_0,\hat{\mathcal E}_1$ & inferred low- and high-margin environment cells \\
$\lambda_t$ & invariance-penalty weight at training step $t$
& $P_q,\mathcal R_q$ & distribution indexed by $q$ and its expected loss \\
$\mathcal R_{\mathrm{tr}},\mathcal R_{\mathrm{wg}}$ & average training risk and worst-group risk
& $\mathcal P_{\mathrm{IRM}},\mathcal P_{\mathrm{REx}}$ & IRM gradient penalty and risk extrapolation (REx) variance penalty \\
$\mathcal E_{\mathrm{tr}}$ & observed training environments used by invariant learning
& $\rho$ & upper bound on the total-variation distance between an oracle group and a matched inferred cell \\
$\mathcal M_{\mathrm{ms}}$ & set of milestone training steps
& $\mathcal C$ & set of candidate encoders retained from Stage~1 \\
$T_{\mathrm{warm}},T_{\mathrm{IRM}}$ & lengths of the warm-up and invariant-training stages
& $\mathsf{mode}$ & environment-update policy used in Algorithm~\ref{alg:prome} \\
$T_{\mathrm{proto}},T_{\mathrm{eval}}$ & intervals between prototype refreshes and validation evaluations
& $\mathrm{WGA}_{\mathrm{val}}$ & validation WGA used to rank repaired encoder--head pairs \\
$\theta_t,h^{(t)}$ & encoder retained at step $t$ and its repaired linear head
& $Q,\pi_{\mathrm{cf}}$ & shortcut status ($Q{=}+1$: aligned; $Q{=}-1$: conflicting) and the sample conflict fraction in Proposition~\ref{prop:Tmargin} \\
\bottomrule
\end{tabular}
\end{table*}
\subsection{Subpopulation Shift}\label{sec:prelim-setting}

Let $\mathcal{D}_{\mathrm{tr}}=\{(x_i,y_i)\}_{i=1}^{n}$ be the observed training set, where $x_i\in\mathcal{X}$ and $y_i\in\{0,1\}$. Each benchmark defines a finite collection $\mathcal G$ of oracle evaluation groups from the target and one or more spurious attributes. Stage~1 observes only $(x_i,y_i)$. Group membership is available in the validation set for model selection and in the test set for evaluation. Let $Q_g$ denote the conditional distribution of group $g$. When $\mathcal G$ partitions the training population, let $\pi_g^{\mathrm{tr}}$ denote its training proportion. The training distribution is then the mixture
\begin{equation}
\mathbb{P}_{\mathrm{tr}}
=
\sum_{g\in\mathcal{G}}
\pi_g^{\mathrm{tr}}Q_g.
\label{eq:train-mixture}
\end{equation}

Under subpopulation shift, the group proportions may change at deployment even when the group-conditional distributions remain fixed. A predictor that performs well on the training mixture can still fail on a rare group whose correlation structure differs from that of the majority~\citep{sagawa2019distributionally}. Average training accuracy hides these failures, and the weakest group decides whether the predictor we finally deploy is usable.

\paragraph{Group robustness.}
For any distribution index $q$, let $P_q$ denote the corresponding distribution over $(X,Y)$. Given a representation $\phi$, a linear classifier $w$, and a surrogate loss $\ell$, define
\begin{equation}
\mathcal{R}_q(\phi,w)
:=
\mathbb{E}_{(X,Y)\sim P_q}
\!\left[
\ell\!\left(
\langle w,\phi(X)\rangle,Y
\right)
\right].
\label{eq:distribution-risk}
\end{equation}

ERM minimizes the average training risk. Group-robust learning instead controls the largest group risk:
\begin{equation}
\begin{aligned}
\mathcal{R}_{\mathrm{tr}}(\phi,w)
&:=
\mathbb{E}_{\mathbb{P}_{\mathrm{tr}}}
\!\left[
\ell\!\left(
\langle w,\phi(X)\rangle,Y
\right)
\right],\\
\mathcal{R}_{\mathrm{wg}}(\phi,w)
&:=
\max_{g\in\mathcal{G}}
\mathcal{R}_g(\phi,w).
\end{aligned}
\label{eq:group-risks}
\end{equation}

The accuracy counterpart of $\mathcal{R}_{\mathrm{wg}}$ is worst-group accuracy
\begin{equation*}
\mathrm{WGA}(h)
:=
\min_{g\in\mathcal{G}}
\Pr_{(X,Y)\sim Q_g}
\!\left[h(X)=Y\right],
\end{equation*}
where $h$ is the hard classifier induced by the prediction score. Both $\mathcal{R}_{\mathrm{wg}}$ and WGA need oracle group identities, which Stage~1 does not have. They remain available for validation and evaluation. Without training-group labels, the learning problem shifts from optimizing known groups to building training environments that can stand in for them. The only signal available for that construction is the geometry of the representation being trained.
\subsection{Problem Formulation}\label{sec:problem}

Standard invariant-learning methods assume access to a partition of the training data into environments $\mathcal{E}_{\mathrm{tr}}$ across which predictive correlations vary. Given these assignments, invariant risk minimization (IRM)~\citep{arjovsky2019invariant} learns a representation $\phi$ and classifier $w$ using the scale-gradient objective
\begin{equation}
\begin{aligned}
\mathcal{L}_{\mathrm{IRM}}
=&
\sum_{e\in\mathcal{E}_{\mathrm{tr}}}
\mathcal{R}_e(\phi,w)\\
&+
\lambda
\sum_{e\in\mathcal{E}_{\mathrm{tr}}}
\left\|
\nabla_{\alpha}
\mathcal{R}_e(\phi,\alpha w)
\big|_{\alpha=1}
\right\|_2^2,
\end{aligned}
\label{eq:irm}
\end{equation}
where $\alpha$ is the scalar dummy classifier and $\lambda\ge 0$ controls the IRMv1 penalty. While well defined when environment assignments are observed, Eq.~\eqref{eq:irm} does not specify how to construct them. To supply the missing assignments, a natural strategy is to read them off a detached reference encoder $\phi_{\mathrm{ref}}$ trained beforehand. However, a partition fixed by $\phi_{\mathrm{ref}}$ reflects the geometry of that encoder, not that of the representation trained downstream. This \emph{environment--representation mismatch} then persists for the rest of training. We call such a construction \emph{exogenous}, and reserve \emph{endogenous} for a partition drawn from a checkpoint on the trajectory that it later regularizes. To remove the mismatch, we make environment construction part of the representation-learning problem rather than a preprocessing step.

Let $\Pi$ map a representation to a two-cell partition of $\mathcal{D}_{\mathrm{tr}}$, and let $\mathcal{P}_{\mathrm{inv}}(\phi,\mathcal{E})$ denote an invariance penalty evaluated on environments $\mathcal{E}$. For compactness, $\mathcal{R}_{\mathrm{tr}}(\phi)$ denotes the training risk of the prediction rule associated with $\phi$. Let $\mathcal C=\{\theta_t\}$ denote the retained encoder candidates, let $f_t$ be the Stage~1 classifier attached to $\theta_t$, and let $\widehat h_t$ be a group-balanced head fitted after freezing $\theta_t$. We formalize the resulting problem as follows.

\begin{definition}[Endogenous environments with repair-aware selection (ERAS)]\label{def:endogenous}
Given a checkpoint $\tilde{\phi}$ on the encoder trajectory, a pair $(\phi^\star,t^\star)$ solves the ERAS problem when the induced partition, optimized representation, and retained candidates arise from the same trajectory and satisfy
\begin{equation}
\begin{aligned}
\hat{\mathcal{E}}
&=
\Pi(\tilde{\phi}),\\
\phi^\star
&\in
\arg\min_{\phi}
\left\{
\mathcal{R}_{\mathrm{tr}}(\phi)
+
\lambda
\mathcal{P}_{\mathrm{inv}}
\bigl(\phi,\hat{\mathcal{E}}\bigr)
\right\}.
\end{aligned}
\label{eq:problem}
\end{equation}
and ranks the retained candidates after classifier repair,
\begin{equation}
t^\star
\in
\arg\max_{t:\theta_t\in\mathcal C}
\mathrm{WGA}_{\mathrm{val}}\bigl(\widehat h_t\circ\theta_t\bigr).
\label{eq:problem-sel}
\end{equation}

The checkpoint $\tilde{\phi}$ lies on the same representation-learning trajectory subsequently optimized under its induced partition. Equation~\eqref{eq:problem-sel} ranks candidate encoders after fitting a matched deployment head to each one.
\end{definition}

Definition~\ref{def:endogenous} aligns candidate generation with the learning trajectory and candidate selection with the repaired predictor used at deployment; neither condition requires training-group labels.

The second condition is decisive because standard two-stage pipelines rank candidates before repair, using the Stage~1 classifier $f_t$:
\begin{equation}
t_{\mathrm{pre}}
\in
\arg\max_{t:\theta_t\in\mathcal C}
\mathrm{WGA}_{\mathrm{val}}\bigl(f_t\circ\theta_t\bigr).
\label{eq:selection-mismatch}
\end{equation}
This criterion mixes representation quality with the bias of $f_t$, and it can discard an encoder that would have repaired well. This motivates comparing candidates under the repaired classifier used at deployment.

The ideal endogenous objective would update $\Pi(\phi)$ with $\phi$. Equation~\eqref{eq:problem} uses a trajectory checkpoint $\tilde\phi$ to make this coupling tractable during training.

\section{Method}\label{sec:method}\label{sec:proposed-method}
In this section, we present ProME, a two-stage framework for group robustness without training-group labels. ProME instantiates the ERAS problem of Definition~\ref{def:endogenous}. It derives environments from the representation being trained, and it ranks candidates after classifier repair.

\begin{figure*}
  \centering
  \includegraphics[width=\textwidth]{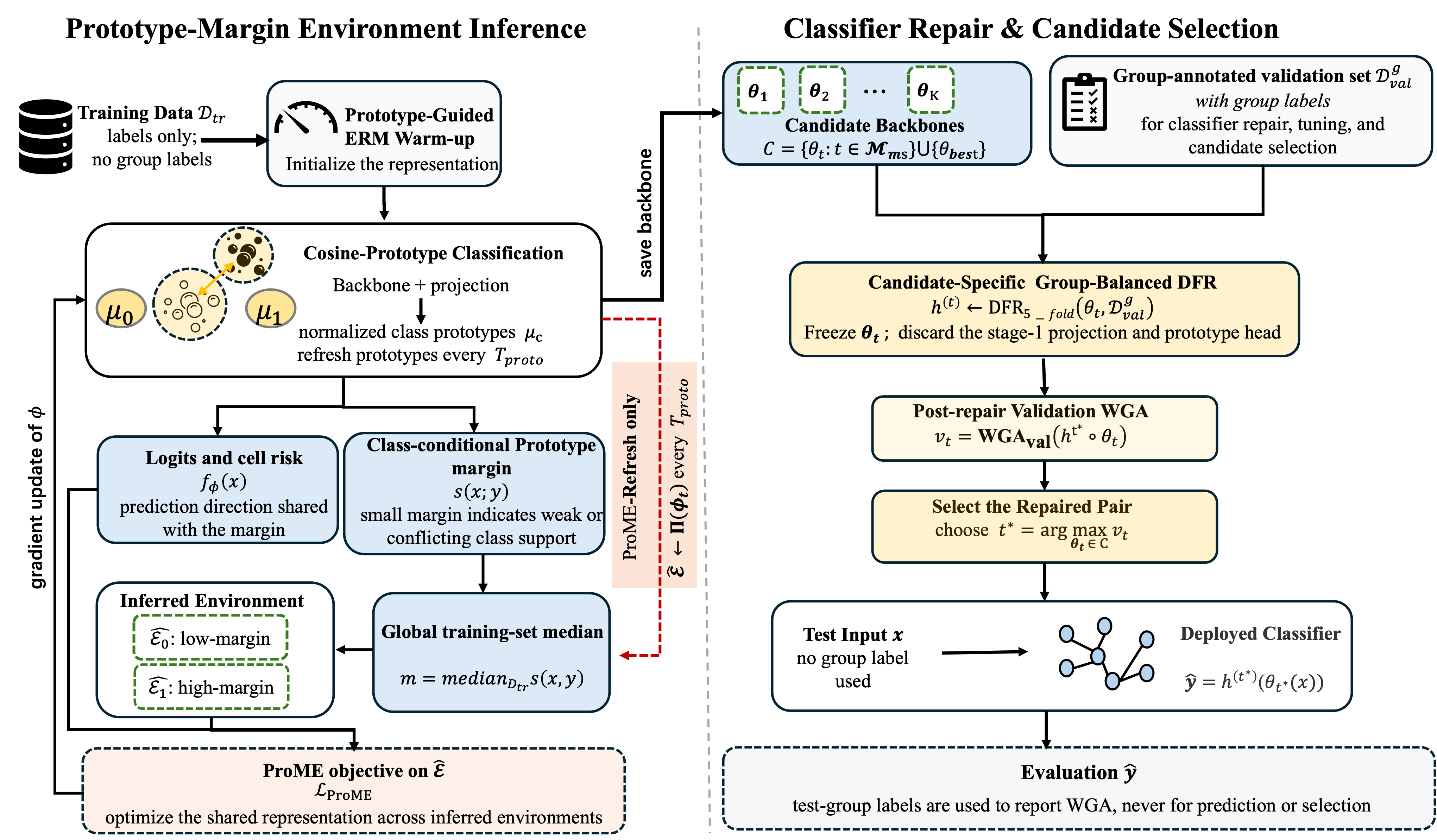}
  \caption{\textbf{Overview of ProME.} Stage~1 learns representation candidates from prototype-margin environments without training-group labels. Stage~2 repairs and selects candidates using group-annotated validation data, with five-fold cross-validation for regularization tuning. The dashed path denotes ProME-Refresh.}
  \label{fig:framework}
\end{figure*}

Figure~\ref{fig:framework} separates candidate generation from deployment-aligned selection. Stage~1 performs a prototype-guided ERM warm-up, partitions the training data by the median prototype margin, and learns across the induced low- and high-margin environments. Checkpoints retained along this trajectory form the candidate set. Stage~2 removes each Stage~1 head, fits a matched group-balanced linear head to every frozen encoder, and ranks the repaired encoder--head pairs by validation WGA. The search is performed offline; test-time inference uses only the selected pair.

\subsection{Cosine-Prototype Classification}\label{sec:proto-irm}
ProME requires an environment signal that follows the same representation geometry used for prediction. We adopt a cosine-prototype classifier to provide this shared geometric basis.
Let $b_{\xi}$ denote the backbone and $p_{\omega}$ denote the projection layer. The normalized representation is defined as

\begin{equation}
\phi_{\xi,\omega}(x)
=
\mathrm{normalize}
\!\left(
p_{\omega}(b_{\xi}(x))
\right),
\label{eq:normalized-representation}
\end{equation}
where $\mathrm{normalize}(u)=u/\|u\|_2$. For each class $c$, ProME computes the class prototype from the current representation:
\begin{align}
\mu_c
&\leftarrow
\mathrm{normalize}
\!\left(
\frac{1}{|\mathcal{D}_c|}
\sum_{(x_i,y_i)\in\mathcal{D}_c}
\phi(x_i)
\right),
\nonumber\\
\mathcal{D}_c
&=
\{(x,y)\in\mathcal{D}_{\mathrm{tr}}:y=c\}.
\label{eq:proto-update}
\end{align}
Given the two prototypes $\mu_0$ and $\mu_1$, the prediction score is defined as
\begin{equation}
f_{\phi}(x;\mu_0,\mu_1)
=
\frac{1}{\tau}
\left[
\cos(\phi(x),\mu_1)
-
\cos(\phi(x),\mu_0)
\right],
\label{eq:proto-logit}
\end{equation}
where $\tau>0$ is the temperature. The sign of $f_{\phi}(x)$ determines the binary prediction, and its magnitude measures the difference in cosine support for the two classes. This makes the prediction rule directly interpretable through the feature--prototype geometry.

\paragraph{Prototype updates.} ProME treats prototypes as non-trainable state variables that are periodically refreshed during training. Every \(T_{\mathrm{proto}}\) steps, we update the prototypes using the closed-form class-mean rule in Eq.~\eqref{eq:proto-update} and stop gradients through the prototype computation. During back-propagation, the refreshed prototypes remain fixed, and gradients update only \(b_{\xi}\) and \(p_{\omega}\) through \(f_{\phi}\). This construction follows the class-mean geometry of prototypical classifiers~\citep{snell2017prototypical}. The alignment between normalized class means and terminal classifier geometry is further supported by neural-collapse results~\citep{papyan2020prevalence,chen2024neuralcollapse}.

The updated prototypes provide a geometric measure of label support for each example. This measure defines the partitioning signal used in the next module.
\subsection{Prototype-Margin Environment Inference}\label{sec:proto-margin}
Without training-group labels, ProME constructs environments from the geometric signals available in the learned representation. Under the shortcut-first hypothesis, examples violating the dominant spurious correlation receive weaker support for their observed labels after ERM warm-up. We quantify this support through the prototype margin:
\begin{equation}
s(x;y,\phi,\mu_0,\mu_1)
=
\cos(\phi(x),\mu_y)
-
\cos(\phi(x),\mu_{1-y}).
\label{eq:proto-margin-score}
\end{equation}

The margin measures the relative prototype support of the observed class. A larger value indicates stronger alignment with the observed label, while a smaller value indicates ambiguous or conflicting class support. The margin is computed from the same classifier used for prediction, which provides a unified geometric basis for environment inference.

\paragraph{Environment construction.}
After prototype-guided ERM warm-up, we compute the margins over $\mathcal{D}_{\mathrm{tr}}$ and construct two environments through a global median split:
\begin{align}
\hat{\mathcal{E}}_0
&=
\{(x,y):s(x;y)\le m\},
\nonumber\\
\hat{\mathcal{E}}_1
&=
\{(x,y):s(x;y)>m\},
\nonumber\\
m
&=
\mathrm{median}_{\mathcal{D}_{\mathrm{tr}}}\,s.
\label{eq:proto-margin-split}
\end{align}

The median split produces two balanced cells up to ties without requiring group proportions. The low-margin environment contains examples with weaker prototype support for their observed labels, while the high-margin environment contains examples with stronger class agreement. Section~\ref{sec:diagnostic} measures shortcut-conflict enrichment and downstream WGA. Corollary~\ref{cor:T2-oracle} separately uses the total-variation quantity $\rho$ to state its transfer condition.

\paragraph{Invariant objective.}
Let $K=|\hat{\mathcal{E}}|$. Over the inferred environments, we optimize the average environment risk together with the IRMv1 and REx penalties:

\begin{align}
\bar{\mathcal{R}}(f_{\phi})
&=
K^{-1}
\sum_{\hat e\in\hat{\mathcal{E}}}
\mathcal{R}_{\hat e}(f_{\phi}),
\nonumber\\
\mathcal{P}_{\mathrm{IRM}}(f_{\phi})
&=
K^{-1}
\sum_{\hat e\in\hat{\mathcal{E}}}
\left\|
\nabla_{w'}
\mathcal{R}_{\hat e}(w'f_{\phi})
\big|_{w'=1}
\right\|_2^2,
\nonumber\\
\mathcal{P}_{\mathrm{REx}}(f_{\phi})
&=
K^{-1}
\sum_{\hat e\in\hat{\mathcal{E}}}
\left(
\mathcal{R}_{\hat e}(f_{\phi})
-
\bar{\mathcal{R}}(f_{\phi})
\right)^2.
\label{eq:prome-penalties}
\end{align}
The Stage~1 objective is
\begin{equation}
\mathcal{L}_{\mathrm{ProME}}(\phi)
=
\bar{\mathcal{R}}(f_{\phi})
+
\lambda_t
\left(
\mathcal{P}_{\mathrm{IRM}}(f_{\phi})
+
\mathcal{P}_{\mathrm{REx}}(f_{\phi})
\right),
\label{eq:prome-irm}
\end{equation}
where $w'$ denotes the scalar dummy classifier in IRMv1, and $\lambda_t$ controls the invariance regularization strength. The IRM penalty enforces consistent classifier optimality across environments, while the REx penalty reduces risk variation across inferred environments~\citep{krueger2021out}. Section~\ref{sec:theory} bounds the gap between average and worst inferred-environment risk through the REx penalty (Proposition~\ref{prop:T2}). IRMv1 adds a complementary constraint on environment-wise classifier optimality.
\paragraph{Partition updates.}
Both ProME variants refresh the prototypes every $T_{\mathrm{proto}}$ steps. The standard protocol constructs $\hat{\mathcal{E}}$ from the post-warm-up encoder and keeps the partition fixed during invariant training. ProME-Refresh updates $\hat{\mathcal{E}}$ by recomputing Eq.~\eqref{eq:proto-margin-split} after each prototype refresh. Both variants optimize the same objective and differ only in the update frequency of the representation-induced partition.

\paragraph{Prototype refinement.}
For the label-augmented sensitivity analysis in Section~\ref{sec:robustness_analysis}, we replace Eq.~\eqref{eq:proto-update} with group-balanced prototypes:
\begin{align}
\mu_c^{\mathrm{bal}}
&\leftarrow
\mathrm{normalize}
\!\left(
\frac{1}{2}\bar{\phi}_{c,0}
+
\frac{1}{2}\bar{\phi}_{c,1}
\right),
\nonumber\\
\bar{\phi}_{c,a}
&=
\frac{1}{|\mathcal{D}_{c,a}|}
\sum_{(x,y)\in\mathcal{D}_{c,a}}
\phi(x),
\label{eq:balanced-proto}
\end{align}
where $\mathcal{D}_{c,a}=\{(x,y)\in\mathcal{D}_{\mathrm{tr}}:y=c,\ A(x)=a\}$. This construction balances the contribution of spurious attributes within each class for sensitivity analysis. ProME uses class labels only to compute prototypes in the standard setting, following Eq.~\eqref{eq:proto-update}.
\subsection{Classifier Repair and Candidate Selection}\label{sec:candidates}
Stage~2 makes classifier repair the basis of representation selection. Given the retained backbones and group-annotated validation data, ProME fits the same group-balanced linear-head family to every candidate and selects the repaired encoder--head pair with the highest validation WGA. Validation groups serve for candidate tracking, head fitting, hyperparameter tuning, and final model selection, without updating the Stage~1 representation.

\paragraph{Candidate construction.}
Pre-repair performance is treated as one source of candidates, not as the deployment decision. ProME retains frozen backbones at fixed milestones $\mathcal M_{\mathrm{ms}}$ and includes $\theta_{\mathrm{best}}$, the checkpoint with the highest pre-repair validation WGA under the prototype classifier. All candidates are subsequently compared only after matched repair:
\begin{equation}
\mathcal C
=
\{\theta_t:t\in\mathcal M_{\mathrm{ms}}\}
\cup
\{\theta_{\mathrm{best}}\}.
\label{eq:candidate-set}
\end{equation}

Each $\theta_t$ is a frozen backbone checkpoint at step $t$. Before classifier repair, ProME removes the Stage~1 projection and cosine-prototype head. This construction preserves representation diversity and evaluates all candidates under the same downstream classifier family.

\paragraph{Candidate-specific classifier repair.}
For each $\theta_t\in\mathcal C$, we freeze the encoder and fit a linear DFR head $h^{(t)}$ on the group-annotated validation set~\citep{kirichenko2023last}. Let $\mathcal D_{\mathrm{val}}^g$ denote the validation examples in oracle group $g\in\mathcal G$. For a regularization strength $\gamma$, the group-balanced objective is
\begin{equation}
\mathcal L_{\mathrm{DFR}}^{(t)}(h;\gamma)
=
\frac{1}{|\mathcal G|}
\sum_{g\in\mathcal G}
\frac{1}{|\mathcal D_{\mathrm{val}}^g|}
\sum_{(x,y)\in\mathcal D_{\mathrm{val}}^g}
\ell\!\left(h(\theta_t(x)),y\right)
+
\gamma\|h\|_2^2.
\label{eq:dfr-objective}
\end{equation}

The objective assigns equal weight to each validation group regardless of its empirical frequency. For each candidate, we select $\gamma$ by five-fold cross-validation and refit the head with the selected value. Freezing the encoder and using the same classifier family make the repaired candidates directly comparable.

\paragraph{Validation selection.}
After fitting all candidate-specific heads, ProME evaluates each encoder--head pair using validation WGA:

\begin{equation}
v_t
=
\mathrm{WGA}_{\mathrm{val}}
\!\left(h^{(t)}\circ\theta_t\right),
\qquad
t^\star
=
\underset{\theta_t\in\mathcal C}{\arg\max}\;v_t.
\label{eq:candidate-selection}
\end{equation}

The selected predictor is $(\theta_{t^\star},h^{(t^\star)})$. The group-annotated validation set is used for Stage~2 candidate tracking, head fitting, regularization tuning, and encoder--head ranking, whereas Stage~1 requires no training-group labels. Section~\ref{sec:post_repair} examines how candidate coverage and matched repair affect post-repair evaluation.

\paragraph{Test-time prediction.}
Given a test input $x$, ProME applies the selected encoder and linear head:
\begin{equation}
\hat y
=
h^{(t^\star)}
\!\left(\theta_{t^\star}(x)\right).
\label{eq:test-prediction}
\end{equation}

We tune the regularization strength by five-fold cross-validation on the group-annotated validation set, refit each candidate-specific head on the full set, and rank the repaired encoder--head pairs by validation WGA. Further implementation details are provided in Appendix~\ref{app:tables}.

Notably, ProME offers several practical advantages:
\begin{itemize}
    \item \textbf{Training-group-free candidate generation.} Prototype-margin environments use class labels only and require no auxiliary environment model.
    \item \textbf{Repair-aware representation selection.} Every candidate is evaluated under the matched classifier family used for deployment.
    \item \textbf{Single-model deployment.} Candidate search is offline; inference uses one selected encoder and one linear head.
\end{itemize}

We summarize the complete ProME paradigm in  Algorithm~\ref{alg:prome}. 

\begin{algorithm}[!t]
\caption{ProME training and candidate selection.}
\label{alg:prome}
\small
\begin{algorithmic}[1]
\Require $\mathcal D_{\mathrm{tr}}$, group-annotated $\mathcal D_{\mathrm{val}}$, $\tau$, $\{\lambda_t\}$, $T_{\mathrm{warm}}$, $T_{\mathrm{IRM}}$, $T_{\mathrm{proto}}$, $T_{\mathrm{eval}}$, $\mathcal M_{\mathrm{ms}}$, and environment-update mode $\mathsf{mode}$.
\Ensure Selected encoder--head pair $(\theta_{t^\star},h^{(t^\star)})$.

\Statex \textbf{Stage~1}: Prototype-Margin Environment Learning
\State Initialize backbone $b_\xi$ and projection $p_\omega$.
\State Optimize $(\xi,\omega)$ for $T_{\mathrm{warm}}$ epochs with $f_\phi$.
\State $\{\mu_c\}\leftarrow\mathrm{Refresh}(\phi_{\xi,\omega},\mathcal D_{\mathrm{tr}})$; 
$\hat{\mathcal E}\leftarrow\Pi(\phi_{\xi,\omega})$.
\State $(u_{\mathrm{best}},\theta_{\mathrm{best}})
\leftarrow
(\mathrm{WGA}_{\mathrm{val}}(f_\phi),\mathrm{copy}(b_\xi))$.

\For{$t=1,\ldots,T_{\mathrm{IRM}}$}
    \If{$t>1$ and $t\equiv1\pmod{T_{\mathrm{proto}}}$}
        \State $\{\mu_c\}\leftarrow\mathrm{Refresh}(\phi_{\xi,\omega},\mathcal D_{\mathrm{tr}})$.
        \State $\hat{\mathcal E}\leftarrow\Pi(\phi_{\xi,\omega})$ 
        \textbf{if} $\mathsf{mode}=\mathrm{Refresh}$.
    \EndIf
    \State Update $(\xi,\omega)$ with Eq.~\eqref{eq:prome-irm} over $\hat{\mathcal E}$.
    \If{$t\equiv0\pmod{T_{\mathrm{eval}}}$}
        \State $u_t\leftarrow\mathrm{WGA}_{\mathrm{val}}(f_\phi)$.
        \State $(u_{\mathrm{best}},\theta_{\mathrm{best}})
        \leftarrow
        (u_t,\mathrm{copy}(b_\xi))$
        \textbf{if} $u_t>u_{\mathrm{best}}$.
    \EndIf
    \If{$t\in\mathcal M_{\mathrm{ms}}$}
        \State $\theta_t\leftarrow\mathrm{copy}(b_\xi)$.
    \EndIf
\EndFor

\Statex \textbf{Stage~2: Classifier Repair and Candidate Selection}
\State $\mathcal C\leftarrow
\{\theta_t:t\in\mathcal M_{\mathrm{ms}}\}
\cup
\{\theta_{\mathrm{best}}\}$.
\For{$\theta_t\in\mathcal C$}
    \State $h^{(t)}
    \leftarrow
    \mathrm{DFR}_{\mathrm{5-fold}}
    (\theta_t,\mathcal D_{\mathrm{val}})$.
    \State $v_t
    \leftarrow
    \mathrm{WGA}_{\mathrm{val}}
    (h^{(t)}\circ\theta_t)$.
\EndFor
\State $t^\star\leftarrow\arg\max_{\theta_t\in\mathcal C}v_t$.
\State \Return $(\theta_{t^\star},h^{(t^\star)})$.
\end{algorithmic}
\end{algorithm}

\section{Theoretical Analysis}\label{sec:theory}
This section analyzes the Stage~1 mechanism required by the first condition of Definition~\ref{def:endogenous}. Lemma~\ref{prop:T1} links the prototype margin to the classifier used for training. Proposition~\ref{prop:Tmargin} states when low margins mark shortcut-conflicting examples. Proposition~\ref{prop:T0} bounds how far the induced partition moves when the representation changes. Proposition~\ref{prop:T2} connects the REx term to the worst inferred-environment risk. The deployment-aligned selection condition is evaluated experimentally in Section~\ref{sec:experiments}.

IRM and REx assume that the environments are already given. Identifying invariant structure from the pooled distribution of $(X,Y)$ needs an additional inductive bias. ProME supplies that bias through prototype-margin geometry. Section~\ref{sec:related} places the alternatives in context.

\subsection{When prototype margins expose shortcut conflicts}\label{sec:theory-margin}

For binary labels, write $\widetilde Y=2Y-1\in\{-1,+1\}$. Let $\widetilde A\in\{-1,+1\}$ encode a binary spurious attribute and define $Q=\widetilde Y\widetilde A$. Here $Q=+1$ denotes shortcut-aligned examples and $Q=-1$ denotes shortcut-conflicting examples. This notation is used only for analysis. ProME never observes $A$ during Stage~1. The margin and the training logit must first be shown to share one direction.

\begin{lemma}[Prototype-margin identity]\label{prop:T1}
Under Assumption~\ref{ass:proto}, let $w^\star=\mu_1-\mu_0$. For every $(x,y)$,
\begin{equation}
s(x;y)=\widetilde y\,\langle\phi(x),w^\star\rangle,
\qquad \widetilde y=2y-1.
\label{eq:T1-equivalence}
\end{equation}
Moreover, $f_\phi(x)=\tau^{-1}\langle\phi(x),w^\star\rangle$. The environment score and the prediction logit share one linear direction.
\end{lemma}

The proof is given in Appendix~\ref{app:proof-T1}. The identity alone does not imply shortcut-conflict separation. Proposition~\ref{prop:Tmargin} gives sufficient conditions under which shortcut-conflicting examples receive lower margins.

\begin{proposition}[Shortcut-conflict separation]\label{prop:Tmargin}
Suppose the signed representation after ERM warm-up satisfies
\begin{equation}
\widetilde Y\phi(X)=u+Qv+\xi,
\qquad u^\top v=0.
\label{eq:shortcut-model}
\end{equation}
Here, $Q=+1$ denotes shortcut alignment and $Q=-1$ denotes shortcut conflict. Let the prototype contrast be $w^\star=\alpha u+\beta v$, where $\alpha,\beta>0$. Assume $\mathbb P(Q=q)>0$ for each $q\in\{-1,+1\}$. If $\mathbb E[\langle\xi,w^\star\rangle\mid Q]=0$, then
\begin{equation}
\mathbb E[s(X;Y)\mid Q=+1]
-
\mathbb E[s(X;Y)\mid Q=-1]
=2\beta\|v\|_2^2.
\label{eq:margin-gap}
\end{equation}
If, in addition, $|\langle\xi,w^\star\rangle|\le\delta$ almost surely for some $0\le\delta<\beta\|v\|_2^2$, every shortcut-conflicting example has a smaller margin than every shortcut-aligned example almost surely. For a sample of size $n$ with distinct margins, let $n_{\mathrm{cf}}$ be its number of conflicting examples and $\pi_{\mathrm{cf}}=n_{\mathrm{cf}}/n$. If $\pi_{\mathrm{cf}}\le 1/2$, the low-margin cell contains all $n_{\mathrm{cf}}$ conflicting examples and has conflict fraction
$n_{\mathrm{cf}}/\lceil n/2\rceil$. This fraction equals $2\pi_{\mathrm{cf}}$ when $n$ is even.
\end{proposition}

The proof is given in Appendix~\ref{app:proof-Tmargin}. The margin is informative when ERM learns a shortcut direction and the projected noise preserves its separation. Proposition~\ref{prop:Tmargin} states as a condition what error- and bias-based methods observe empirically. Section~\ref{sec:diagnostic} tests it on real data.

\subsection{Stability of the endogenous partition}\label{sec:theory-stability}

ProME computes both prototypes and the median on the training sample. A useful stability result must account for this data dependence. Let $\hat\Pi_n(\phi)$ denote the empirical split in Eq.~\eqref{eq:proto-margin-split}. Consistent with that equation, define $\hat e_i(\phi)=0$ when $s_i(\phi)\le\hat m_n(\phi)$ and $\hat e_i(\phi)=1$ otherwise.

\begin{proposition}[Same-sample partition stability]\label{prop:T0}
Fix a labeled sample $\mathcal D_{\mathrm{tr}}=\{(x_i,y_i)\}_{i=1}^n$ with $|\mathcal D_c|>0$ for both classes. Let $\phi$ and $\phi'$ be unit-normalized representations whose class prototypes are computed from this same sample by Eq.~\eqref{eq:proto-update}. Define
\begin{align*}
\varepsilon&=\max_i\|\phi(x_i)-\phi'(x_i)\|_2,\\
\beta_n&=\min_{c,\psi\in\{\phi,\phi'\}}
\left\|\frac{1}{|\mathcal D_c|}\sum_{i:y_i=c}\psi(x_i)\right\|_2,
\qquad
L_s=2+\frac{4}{\beta_n}.
\end{align*}
Assume $\beta_n>0$. Let $\hat m_n(\phi)$ be the empirical median and define the empirical concentration function
\begin{equation*}
\hat\omega_\phi(t)=\frac{1}{n}\sum_{i=1}^n
\ind\bigl[|s_i(\phi)-\hat m_n(\phi)|\le t\bigr].
\end{equation*}
Then the prototype, margin, and median perturbations satisfy
\begin{align}
\max_c\|\mu_c(\phi)-\mu_c(\phi')\|_2
&\le \frac{2\varepsilon}{\beta_n},\nonumber\\
\max_i|s_i(\phi)-s_i(\phi')|
&\le L_s\varepsilon,\nonumber\\
|\hat m_n(\phi)-\hat m_n(\phi')|
&\le L_s\varepsilon.
\label{eq:T0-stability}
\end{align}
Consequently, the fraction of reassigned examples obeys
\begin{equation}
\frac{1}{n}\sum_{i=1}^n
\ind[\hat e_i(\phi)\neq\hat e_i(\phi')]
\le
\hat\omega_\phi(2L_s\varepsilon).
\label{eq:T0-reassignment}
\end{equation}
\end{proposition}

The proof is given in Appendix~\ref{app:proof-T0}. Proposition~\ref{prop:T0} is deterministic and covers the implementation that ProME actually runs, in which prototypes and the median come from the same training set. A concentration argument with independently fixed prototypes would not. Equation~\eqref{eq:T0-reassignment} also identifies the failure modes. Instability increases when class means approach zero, which enlarges $L_s$, or when many margins accumulate near the median, which enlarges $\hat\omega_\phi$.

\subsection{Risk control on inferred environments}\label{sec:theory-risk}

Margin separation and partition stability explain how the environments are formed, but neither result alone guarantees robustness. We next characterize the quantity controlled directly by the REx penalty in Eq.~\eqref{eq:prome-penalties}. For a fixed classifier $f$, write $r_e=\mathcal R_{\hat e}(f)$, $\bar{\mathcal R}=K^{-1}\sum_e r_e$, and $\mathcal R_{\mathrm{wg}}^{(\hat{\mathcal E})}=\max_e r_e$.

\begin{proposition}[Worst inferred-environment risk bound]\label{prop:T2}
For any $K\ge2$ inferred environments,
\begin{equation}
\mathcal R_{\mathrm{wg}}^{(\hat{\mathcal E})}(f)
\le
\bar{\mathcal R}(f)
+
\sqrt{(K-1)\mathcal P_{\mathrm{REx}}(f)}.
\label{eq:T2-bound}
\end{equation}
For the two-cell partition used by ProME, equality holds.
\end{proposition}

The proof is given in Appendix~\ref{app:proof-T2}. Equation~\eqref{eq:T2-bound} is a deterministic inequality for a fixed predictor, partition, and vector of cell risks. Both terms on its right are aligned with the Stage~1 objective, but the inequality does not by itself bound the empirical-to-population gap or guarantee that optimization reaches a small right-hand side. It requires no calibration assumption linking the IRMv1 gradient penalty to worst-cell risk. IRMv1 remains a complementary optimization signal, while REx supplies the direct risk-dispersion control.

We next relate inferred-cell risk to oracle-group risk under an alignment condition.
\begin{corollary}[Conditional transfer to oracle groups]\label{cor:T2-oracle}
Assume the loss takes values in $[0,B]$. Condition on the learned predictor and partition, and define $\mathrm{TV}(P,Q)=\sup_A|P(A)-Q(A)|$ for distributions over $(X,Y)$. If every oracle group distribution $Q_g$ is within total variation $\rho$ of some inferred-cell distribution $Q_{e(g)}$, then
\begin{equation}
\mathcal R_{\mathrm{wg}}^{(\mathcal G)}(f)
\le
\bar{\mathcal R}(f)
+
\sqrt{(K-1)\mathcal P_{\mathrm{REx}}(f)}
+B\rho.
\label{eq:T2-oracle}
\end{equation}
\end{corollary}

The same appendix section proves the corollary. Corollary~\ref{cor:T2-oracle} controls oracle-group risk only under the stated alignment condition. For the $0$--$1$ loss, where $B=1$, it equivalently yields
$\mathrm{WGA}(f)\ge 1-\bar{\mathcal R}(f)-\sqrt{(K-1)\mathcal P_{\mathrm{REx}}(f)}-\rho$. The experiments measure shortcut-conflict enrichment; they do not estimate $\rho$.

Together, the results characterize three Stage~1 properties under explicit conditions: shortcut-conflict separation, same-sample partition stability, and fixed-risk control across inferred environments. Classifier repair and repair-aware selection remain outside the analysis.

\section{Experiments}\label{sec:experiments}
In this section, we present the experimental results to evaluate two claims: \textbf{(i)} ProME improves worst-group accuracy without training-group labels; and \textbf{(ii)} prototype margins enrich shortcut-conflicting examples, while classifier repair reshapes candidate evaluation.
\subsection{Setup}
\paragraph{Datasets.}
We evaluate ProME on four subpopulation-shift benchmarks, each pairing a label with a spurious attribute: Waterbirds~\citep{sagawa2019distributionally}, bird type against background; CelebA~\citep{liu2015deep}, blond hair against gender; CivilComments~\citep{borkan2019nuanced}, toxicity across identity-based subpopulations; and ColoredMNIST~\citep{arjovsky2019invariant}, digit parity against a controllable color shortcut. The first three datasets form the main evaluation, while ColoredMNIST provides a controlled test under shortcut reversal.

\paragraph{Baselines.}
We group the baselines by their access to group labels. When training-group labels are available, Group DRO~\citep{sagawa2019distributionally} minimizes the worst-group loss, while RWG~\citep{idrissi2022simple} balances group sampling. Without validation-group access, ERM~\citep{sagawa2020investigation} minimizes average loss, while XRM+GroupDRO~\citep{pezeshki2024discovering} infers groups without oracle group labels. EIIL~\citep{creager2021environment}, JTT~\citep{liu2021just}, CnC~\citep{zhang2022correct}, and AFR~\citep{qiu2023simple} use validation groups only for model selection. When validation groups can also be used for fitting, SSA~\citep{nam2022spread} infers training-group labels, MAPLE~\citep{zhou2022model} learns weights for training examples, and DFR~\citep{kirichenko2023last} fits a balanced classifier on a frozen representation. GSR-HF and GSR~\citep{qiao2025group} use influence scores to reweight examples for last-layer retraining. This grouping makes the supervision budget of each comparison explicit.

\paragraph{Implementation and Evaluation.}
We use ImageNet-pretrained ResNet-50 encoders for Waterbirds and CelebA, BERT-base-uncased for CivilComments, and the three-layer convolutional network of IRM~\citep{arjovsky2019invariant} for ColoredMNIST. The candidate pool contains fixed milestone checkpoints and the validation-best pre-repair encoder. We freeze each candidate and fit a group-balanced logistic-regression head, selecting its $\ell_2$ penalty by five-fold cross-validation. We report WGA over the oracle test groups. Waterbirds, CelebA, and ColoredMNIST contain four $(y,a)$ groups, while CivilComments uses the standard $16$ identity--label groups. The headline results report the mean and population standard deviation over seeds $\{42,0,1\}$, and we also evaluate a fixed pool of ten seeds. Table~\ref{tab:main} uses published baseline results from their original protocols, while the source, penalty, and candidate-set comparisons use matched settings; label-augmented diagnostics are identified explicitly. Appendix~\ref{app:tables} provides the remaining training details. Because baseline entries follow their published protocols, Table~\ref{tab:main} provides reported rather than controlled comparisons.

\subsection{Main Results}\label{sec:main_results}

\subsubsection{Overall Group Robustness}

\begin{table*}
\centering
\caption{\textbf{Performance comparison on three real-world benchmarks.} For ProME, we report the \textbf{mean and standard deviation ($\pm$ std)} of WGA over seeds $\{42,0,1\}$; baseline entries are taken from their original protocols. \textit{Train} indicates whether oracle training-group labels update the representation, and \textit{Val} indicates oracle validation-group use: S for selection or hyperparameter tuning, F+S for classifier fitting and selection, and -- for none. Rows are grouped by increasing group-label access. Average is the unweighted mean across datasets. \textbf{bolded} and \underline{underlined} values mark the highest and second-highest reported means.}
\label{tab:main}
\small
\begin{tabular*}{\textwidth}{@{\extracolsep{\fill}}lcccccc@{}}
\toprule
\multirow{2}{*}{Method} & \multicolumn{2}{c}{Group-label} & \multicolumn{4}{c}{Worst-group accuracy (\%)} \\
\cmidrule(lr){2-3}\cmidrule(lr){4-7}
 & Train & Val & Waterbirds & CelebA & CivilComments & Average \\
\midrule
\multicolumn{7}{@{}l}{\textit{Training-group supervision (reference)}} \\
\hline
Group DRO~\citep{sagawa2019distributionally} & Yes & S & $91.4{\pm}1.1$ & $88.9{\pm}2.3$ & $70.0{\pm}2.0$ & $83.4$ \\
RWG~\citep{idrissi2022simple} & Yes & S & $87.6{\pm}1.6$ & $84.3{\pm}1.8$ & $72.0{\pm}1.9$ & $81.3$ \\
\midrule
\multicolumn{7}{@{}l}{\textit{No group labels}} \\
\hline
ERM~\citep{sagawa2020investigation} & No & -- & $74.9{\pm}2.4$ & $46.9{\pm}2.8$ & $55.6{\pm}0.6$ & $59.1$ \\
XRM+GroupDRO~\citep{pezeshki2024discovering} & No & -- & $88.1$ & $89.1$ & $\underline{72.2}$ & $83.1$ \\
\midrule
\multicolumn{7}{@{}l}{\textit{Validation groups for selection}} \\
\hline
EIIL~\citep{creager2021environment} & No & S & $78.7$ & $83.3$ & $67.0$ & $76.3$ \\
JTT~\citep{liu2021just} & No & S & $86.7$ & $81.1$ & $69.3$ & $79.0$ \\
CnC~\citep{zhang2022correct} & No & S & $88.5{\pm}0.3$ & $88.8{\pm}0.9$ & $68.9{\pm}2.1$ & $82.1$ \\
AFR~\citep{qiu2023simple} & No & S & $90.4{\pm}1.1$ & $82.0{\pm}0.5$ & $68.7{\pm}0.6$ & $80.4$ \\
\midrule
\multicolumn{7}{@{}l}{\textit{Validation groups for fitting and selection}} \\
\hline
SSA~\citep{nam2022spread} & No & F+S & $89.0{\pm}0.6$ & $\mathbf{89.8}{\pm}1.3$ & $69.9{\pm}2.0$ & $82.9$ \\
MAPLE~\citep{zhou2022model} & No & F+S & $91.7$ & $88.0$ & $64.1$ & $81.3$ \\
DFR~\citep{kirichenko2023last} & No & F+S & $\underline{92.9{\pm}0.2}$ & $88.3{\pm}1.1$ & $70.1{\pm}0.8$ & $83.8$ \\
GSR-HF~\citep{qiao2025group} & No & F+S & $87.5{\pm}0.1$ & $86.3{\pm}0.4$ & $68.9{\pm}0.2$ & $80.9$ \\
GSR~\citep{qiao2025group} & No & F+S & $\underline{92.9{\pm}0.0}$ & $87.0{\pm}0.4$ & $71.7{\pm}0.6$ & $\underline{83.9}$ \\
\textbf{ProME (ours)} & No & F+S & $\mathbf{93.1}{\pm}0.3$ & $\underline{89.3{\pm}0.5}$ & $\mathbf{78.7}{\pm}0.2$ & $\mathbf{87.0}$ \\
\bottomrule
\end{tabular*}
\end{table*}

\paragraph{ProME achieves the highest average WGA among methods compared at the same group-label access.} ProME uses validation groups for classifier fitting and selection, matching the access of SSA, MAPLE, DFR, GSR-HF, and GSR. Its average WGA is $87.0\%$, $3.1$ points above GSR ($83.9\%$), the highest baseline entry at this access level. ProME ranks first on Waterbirds ($93.1\%$) and CivilComments ($78.7\%$), and second on CelebA ($89.3{\pm}0.5\%$). The CivilComments result is $6.5$ points above the strongest prior entry in the table ($72.2\%$). These differences compare published results under their original protocols. The following analyses use matched settings to examine the two alignment mechanisms.

\subsubsection{Prototype Margins Enrich Shortcut Conflicts}
\label{sec:diagnostic}
Definition~\ref{def:endogenous} draws the partition from a checkpoint on the trajectory it later regularizes. This subsection examines the induced environments through conflict enrichment and downstream controls over the partition signal, its source encoder, and the invariance penalty.

\begin{figure*}
  \centering
  \includegraphics[width=1.0\textwidth]{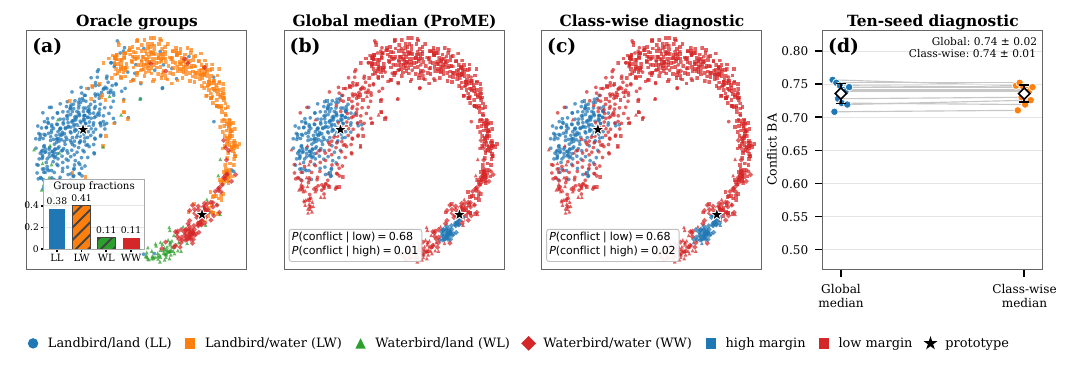}
  \caption{\textbf{Prototype-margin enrichment of shortcut conflicts on Waterbirds.} Panels~\textbf{(a--c)} show the same seed-$42$ test examples in a shared t-distributed stochastic neighbor embedding (t-SNE). \textbf{(a)}~Oracle groups; the inset reports their fractions in the displayed subset. \textbf{(b)}~Environments induced by the global-median rule used by ProME. \textbf{(c)}~Environments induced by a class-wise median diagnostic. Marker shapes denote oracle groups, colors denote inferred environments, and stars denote class prototypes. \textbf{(d)}~Conflict balanced accuracy on the training split over ten seeds. Gray lines connect the same seed; diamonds and error bars show the mean and standard deviation.}
  \label{fig:tsne_wb}
\end{figure*}

\noindent\textbf{Prototype margins concentrate shortcut conflicts in the low-margin environment.} Under the global-median rule, shortcut-conflicting examples constitute $0.68$ of the low-margin environment and $0.01$ of the high-margin environment (Fig.~\ref{fig:tsne_wb}). Fig.~\ref{fig:margin_gallery} shows the same ordering at the example level. In the label-augmented Waterbirds diagnostic, the prototype-margin pipeline reaches $90.31\pm1.29\%$ WGA, compared with $86.76\pm2.30\%$ for the ERM-loss control (Table~\ref{tab:partition_wga}). This auxiliary comparison complements the class-only conflict-enrichment evidence in Fig.~\ref{fig:tsne_wb}. The result is consistent with the conditional separation in Proposition~\ref{prop:Tmargin}. Prototype margins also reach $75.77\%$ on ColoredMNIST, compared with $74.09\%$ for ERM loss. In the label-augmented CelebA diagnostic, the two pipelines yield means that differ by $0.56$ points, with a shared standard deviation of $0.94$ points. Section~\ref{sec:post_repair} examines candidate coverage on CelebA.
\begin{table}
\centering
{
\caption{\textbf{Downstream WGA for prototype-margin and ERM-loss partitions.} This is a pipeline-level comparison: each partition passes through the same Stage~2 repair and selection protocol. For the prototype-margin row, Waterbirds and CelebA are label-augmented diagnostics using the group-balanced prototypes of Eq.~\eqref{eq:balanced-proto}, reported over five and three seeds; ColoredMNIST follows the class-only protocol over ten seeds.}
\label{tab:partition_wga}
\small
\setlength{\tabcolsep}{3.5pt}
\begin{tabular}{lccc}
\toprule
\multirow{2}{*}{Partition} & \multicolumn{3}{c}{Worst-group accuracy (\%)}\\
\cmidrule(lr){2-4}
& Waterbirds & CelebA & ColoredMNIST\\
\midrule
ERM-loss & 86.76\std{2.30} & 85.93\std{0.94}  & 74.09\std{5.75}\\
Prototype-margin     & 90.31\std{1.29} & 85.37\std{0.94} &  75.77\std{5.48}\\
\bottomrule
\end{tabular}}
\end{table}
\begin{figure*}
  \centering
  \rev{\includegraphics[width=1.0\textwidth]{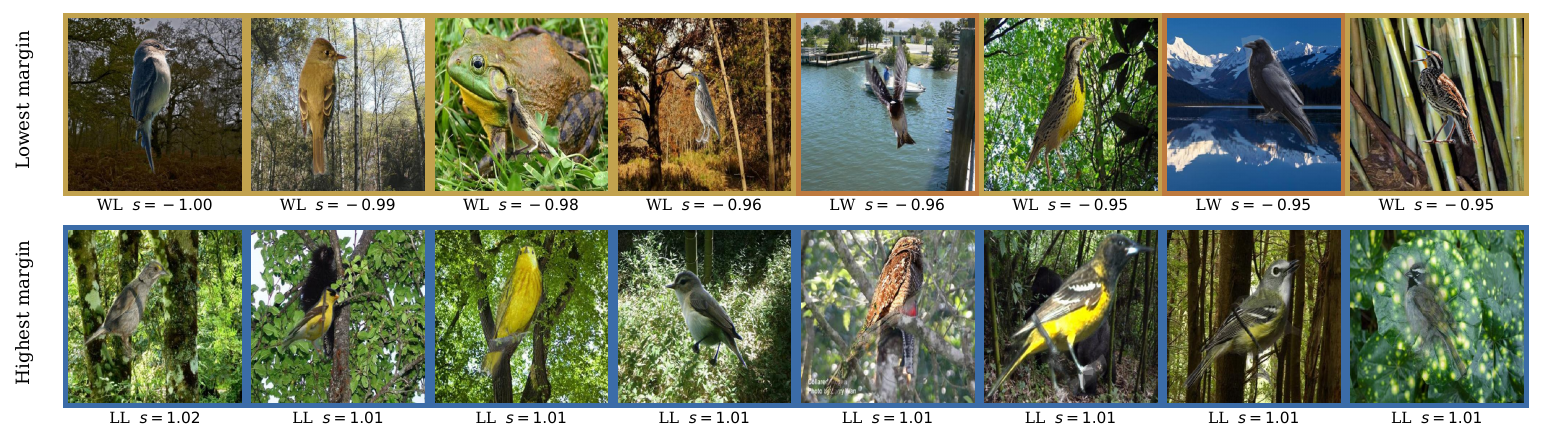}}
  \caption{\textbf{Waterbirds examples ranked by prototype margin.} The displayed lowest-margin examples contain label--background conflicts, whereas the displayed highest-margin examples show aligned pairs. Oracle group labels are included only to support interpretation of the examples.}
  \label{fig:margin_gallery}
\end{figure*}

\begin{table}
\centering
\caption{\textbf{Waterbirds WGA before and after classifier repair.} Each variant changes one element of Stage~1 and then passes through the same Stage~2. \textit{Source} is the encoder the partition is read from. Mean $\pm$ standard deviation over three seeds.}
\label{tab:threearm}
\footnotesize
\setlength{\tabcolsep}{2pt}
\renewcommand{\arraystretch}{1.0}
\begin{tabular}{@{}lccccc@{}}
\toprule
\multirow{2}{*}{Variant} & \multicolumn{2}{c}{Partition} & \multirow{2}{*}{$\lambda$} & \multicolumn{2}{c}{Worst-group accuracy (\%)} \\
\cmidrule(lr){2-3}\cmidrule(lr){5-6}
& Source & Criterion & & Pre-repair & Final \\
\midrule
Live margin & Live & Margin & $3$ & 78.85\std{3.56} & 90.34\std{1.75} \\
Live loss & Live & Loss & $3$ & 81.15\std{1.11} & 91.59\std{1.21} \\
No penalty & Live & Margin & $0$ & 72.95\std{1.02} & 90.71\std{1.66} \\
Frozen reference & Frozen & Margin & $3$ & 68.95\std{3.39} & 90.60\std{0.95} \\
\bottomrule
\end{tabular}
\end{table}

\noindent\textbf{Trajectory-derived margins improve pre-repair WGA in the matched source comparison.} The live-margin and frozen-reference variants use the same margin criterion and invariance weight, but derive margins from different encoders (Table~\ref{tab:threearm}). The live source raises pre-repair WGA from $68.95\pm3.39\%$ to $78.85\pm3.56\%$. After matched classifier repair, their final WGA values are $90.34\pm1.75\%$ and $90.60\pm0.95\%$, respectively. This matched comparison tests the effect of the partition source before repair. Proposition~\ref{prop:T0} separately bounds assignment changes as the representation evolves.

\noindent\textbf{The invariance penalty affects pre-repair WGA, while the global median matches the class-wise diagnostic.} Removing the penalty lowers pre-repair WGA from $78.85\pm3.56\%$ to $72.95\pm1.02\%$ (Table~\ref{tab:threearm}). Their final WGA values after matched repair are $90.34\pm1.75\%$ and $90.71\pm1.66\%$, respectively. The observed penalty effect is concentrated in pre-repair WGA. Replacing the global median with a class-wise median produces similar environment composition ($0.68$ and $0.02$) and conflict balanced accuracy ($0.74\pm0.01$ and $0.74\pm0.02$). The global rule provides comparable diagnostic separation with one threshold.

\subsubsection{Classifier Repair Reshapes Candidate Evaluation}
\label{sec:post_repair}
The second ERAS condition ranks encoder--head pairs by validation WGA after matched classifier repair (Definition~\ref{def:endogenous}). We examine how candidate coverage and classifier repair affect post-repair evaluation.

\noindent\textbf{Multi-candidate selection improves post-repair WGA on CelebA.} Fig.~\ref{fig:threearm_slope}(c) holds the repair procedure fixed and varies only the candidate set. A single retained checkpoint reaches $86.94\%$ WGA after repair. Ranking repaired candidates from milestone and random pools reaches $89.31\%$ and $89.87\%$, respectively. Both pools outperform the single-checkpoint setting in these runs.
\begin{figure*}[!t]
  \centering
  \rev{\includegraphics[width=1.0\textwidth]{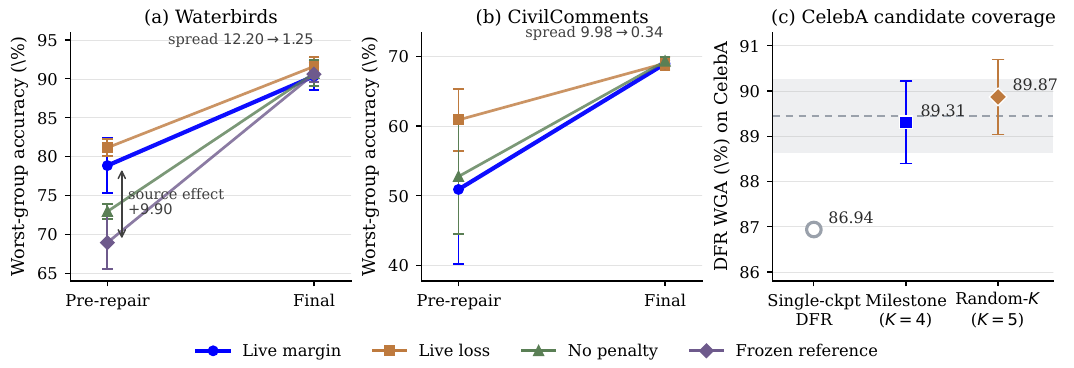}}
\caption{\textbf{Repair-aware candidate evaluation.} \textbf{(a)}~Waterbirds and \textbf{(b)}~CivilComments before and after classifier repair; each line connects one Stage~1 variant of Table~\ref{tab:threearm}. \textbf{(c)}~CelebA with single-checkpoint, milestone, and random candidate sets under a fixed repair procedure; the band marks the full multi-candidate pipeline.}
\label{fig:threearm_slope}
\end{figure*}

\noindent\textbf{Classifier repair narrows performance differences among Stage~1 variants.} Across four Waterbirds variants, pre-repair WGA spans $12.20$ points, whereas matched repair reduces the final spread to $1.25$ points (Table~\ref{tab:threearm}). CivilComments shows the same pattern under the shortened schedule in Fig.~\ref{fig:threearm_slope}(b), where the spread contracts from $9.98$ to $0.34$ points. These contractions show that comparisons based on Stage~1 heads need not persist after repair. Repair-aware selection instead evaluates the encoder--head pairs used at deployment.

\noindent\textbf{Classifier repair reduces a posterior-alignment diagnostic on ColoredMNIST.} Fig.~\ref{fig:posterior_alignment} compares empirical $\mathbb{P}(y{=}1\mid h,e)$ across two training environments and the reversed test environment. ProME has the largest integrated gap before repair and the smallest gap after repair, below ERM and the ERM-loss control. he posterior-alignment gap characterizes how classifier repair changes score geometry under shortcut reversal. Corollary~\ref{cor:T2-oracle} states the transfer condition in terms of the total-variation quantity $\rho$.
\begin{figure*}
  \centering
  \includegraphics[width=0.96\textwidth]{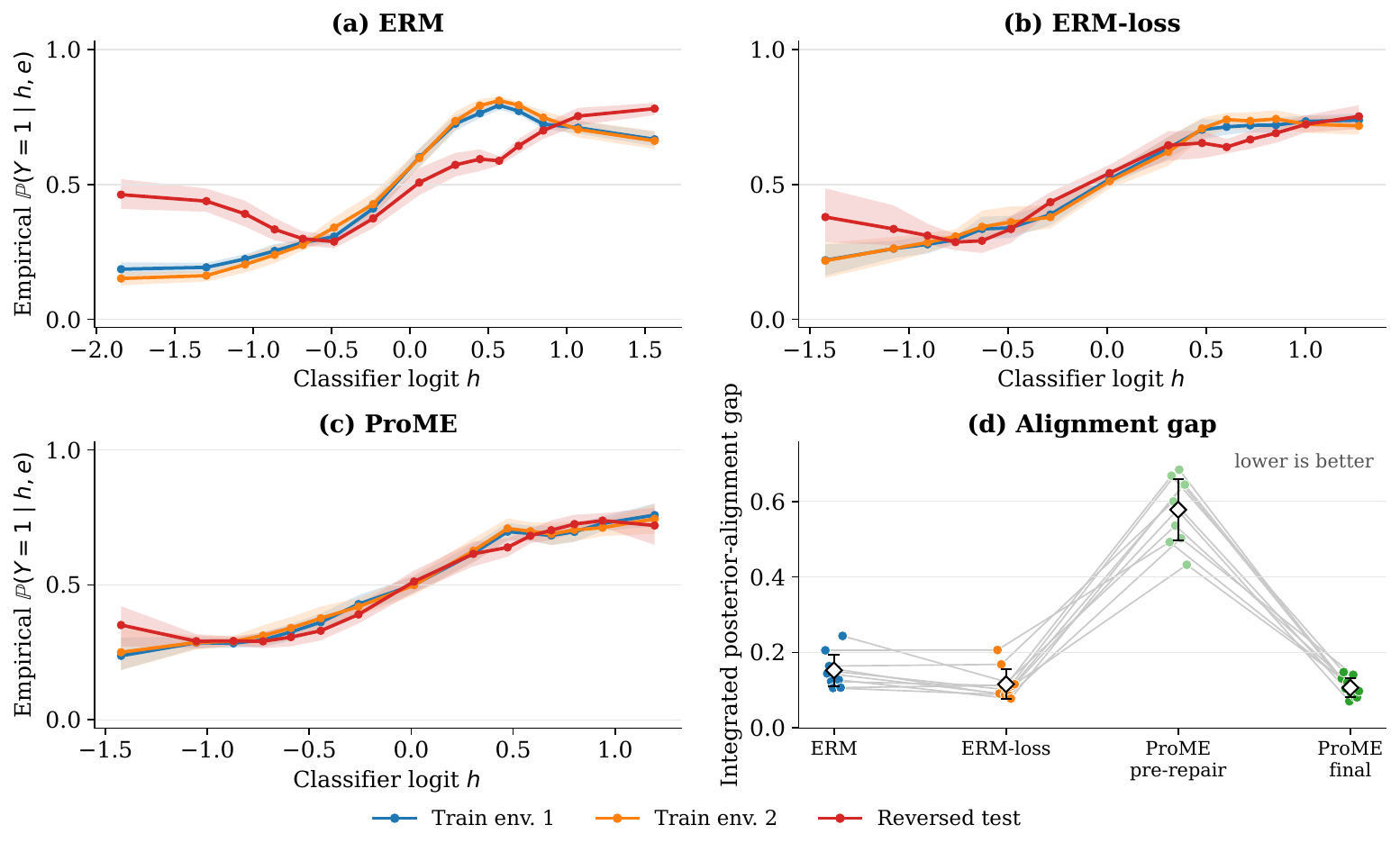}
  \caption{\textbf{Classifier repair reduces the posterior-alignment gap on ColoredMNIST.} Panels (a)--(c) plot the empirical $\mathbb{P}(y{=}1\mid h,e)$ for two training environments and the reversed test environment under ERM, ERM-loss, and ProME after classifier repair. Panel (d) reports the integrated posterior-alignment gap for ERM, ERM-loss, and ProME before and after repair. A lower value indicates better alignment. Gray lines connect the same seed. Curves use $50$ equal-width bins with at least five samples per bin.}
  \label{fig:posterior_alignment}
\end{figure*}
\subsection{Seed Budget, Label Budget, and Hyperparameters}\label{sec:robustness_analysis}
\paragraph{Real-world WGA remains stable across the fixed seed budgets.} Across ten fixed seeds, Waterbirds, CelebA, and CivilComments have standard deviations of $0.62$, $0.51$, and $0.30$ WGA points (Fig.~\ref{fig:stability}). Their ten-seed means remain within $0.62$ points of the three-seed results in Table~\ref{tab:main}. Fig.~\ref{fig:sensitivity}(d) shows little change as the seed budget grows from $3$ to $10$. ColoredMNIST has a wider standard deviation of $5.48$ points, so the stability conclusion is limited to the three real-world benchmarks.

\begin{figure}
  \centering
  \includegraphics[width=\columnwidth]{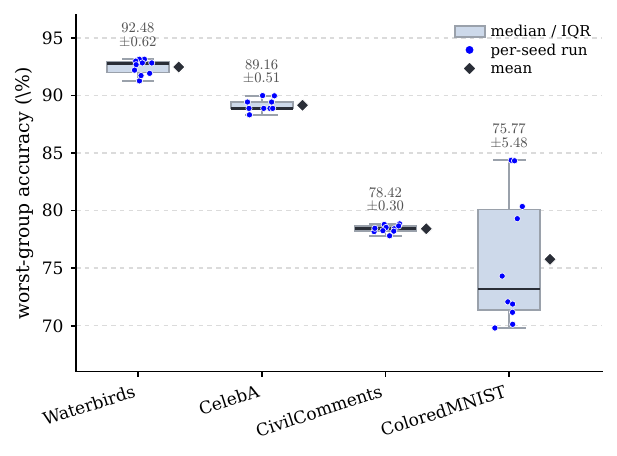}
  \captionof{figure}{\textbf{WGA across ten fixed seeds.} Boxes show interquartile ranges, center lines show medians, points show individual runs, and diamonds show means.}
  \label{fig:stability}
\end{figure}

\paragraph{The value of validation-group labels depends on minority-group support.}
Stage~2 uses group-annotated validation data for classifier repair. Table~\ref{tab:label} holds the ProME representation fixed and changes only the group-balanced data used to fit the head. Val-DFR exceeds train-DFR by $18.68$ points on Waterbirds, whose smallest training group contains $56$ examples. The gap is $1.83$ points on CelebA, whose smallest group contains $1{,}387$ examples. Across six ColoredMNIST settings, the train-DFR minus val-DFR gap approaches zero as $n_{\min}$ increases (Fig.~\ref{fig:cmnist_size}). These results support validation-group labels for classifier repair when training minority support is limited. The source and partition controls above examine the Stage~1 environment mechanism. Stage~1 representation learning uses no training-group labels.
\begin{table}
\centering
\caption{\textbf{Effect of the data used to fit the classifier-repair head.} The \methodname{} representation is fixed; only the group-balanced fitting data changes. Mean $\pm$ standard deviation over ten seeds. $N$ is the training-set size and $n_{\min}$ its smallest $(y,a)$ group.}
\label{tab:label}
\footnotesize
\setlength{\tabcolsep}{4pt}
\begin{tabular}{lccc}
\toprule
 Dataset ($N$) & \makecell{Smallest\\group $n_{\min}$} & train-DFR (\%) & val-DFR (\%)  \\
\midrule
 Waterbirds (4,795)   & 56  & 73.80\std{2.51}  & 92.48\std{0.62} \\
\midrule
 CelebA (162,770) & 1{,}387  & 87.61\std{1.00} & 89.44\std{0.82} \\
\bottomrule
\end{tabular}
\end{table}

\begin{figure}
  \centering
  \rev{\includegraphics[width=0.96\columnwidth]{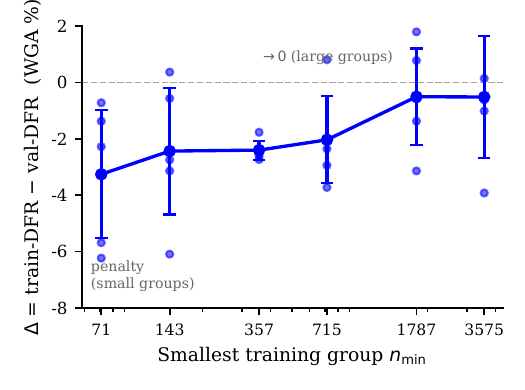}}
  \captionof{figure}{\textbf{Effect of minority support on classifier-repair label source.} train-DFR minus val-DFR WGA versus the smallest training group $n_{\min}$, reported as mean $\pm$ standard deviation over five seeds.}
  \label{fig:cmnist_size}
\end{figure}

\paragraph{The label-augmented Waterbirds diagnostic is stable across the swept settings.}
Fig.~\ref{fig:sensitivity}(a--c) sweeps the invariance weight, the temperature, and the prototype-refresh period on Waterbirds under the label-augmented refinement with group-balanced prototypes. Every swept value of all three hyperparameters stays inside a $0.93$-point band, from $92.14\%$ to $93.07\%$, and this includes the endpoints $\lambda=30$ at ten times the default weight and $T_{\mathrm{proto}}=\infty$, which suspends prototype refreshing altogether. Within this Waterbirds sweep, WGA is insensitive to an order-of-magnitude change in the Stage~1 invariance weight.

\begin{figure}
  \begin{minipage}{1.0\columnwidth}
  \centering
  \includegraphics[width=1.0\columnwidth]{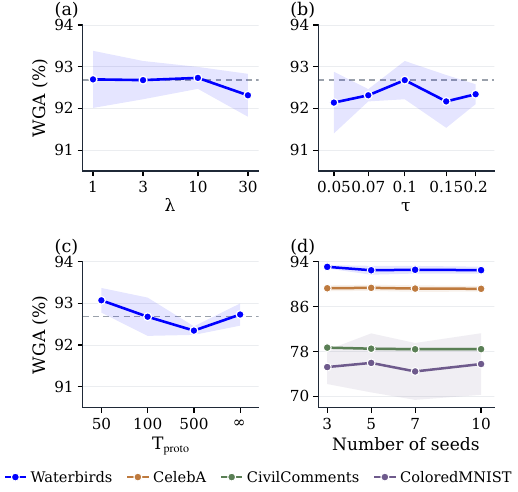}
  \captionof{figure}{\textbf{Hyperparameter and seed-budget sensitivity.} \textbf{(a--c)}~Label-augmented Waterbirds diagnostic across the invariance weight, temperature, and prototype-refresh period; bands show standard deviations and dashed lines mark defaults. \textbf{(d)}~WGA across seed budgets.}
  \label{fig:sensitivity}
  \end{minipage}
\end{figure}


\section{Conclusion}\label{sec:conclusion}
In this paper, we presented ProME, a two-stage and backbone-agnostic framework designed to align environment construction and model selection with the deployed predictor when training-group labels are unavailable. By formulating group robustness without training-group labels as the endogenous environments with repair-aware selection (ERAS) problem, ProME couples trajectory-based environment construction with post-repair comparison of encoder--head pairs. Theoretically, we characterize shortcut-conflict separation and partition stability under explicit conditions, and bound worst inferred-environment risk for a fixed predictor and partition, with transfer to oracle groups under total-variation alignment. Empirically, ProME achieves the highest average WGA among methods under matched group-label access, while analyses show pre-repair gains from trajectory-derived margins and reduced differences among Stage~1 variants after classifier repair. As a backbone-agnostic framework that learns representations without training-group labels and deploys a single encoder--head pair, ProME offers deployment alignment as a practical design principle for group-robust learning, and we hope this work facilitates future research across diverse forms of subpopulation shift.

\section*{CRediT authorship contribution statement}
\textbf{Qianqian Wang:} Conceptualization, Methodology, Software, Validation, Formal analysis, Investigation, Writing -- original draft. 
\textbf{Yunshan Li:} Writing – review \& editing.
\textbf{Dawei Huang:} Software, Resources.
\textbf{Wenwu Gong:} Writing – review \& editing,
Methodology, Conceptualization.
\textbf{Lili Yang:} Supervision,
Resources, Writing – review \& editing.

\section*{Declaration of competing interest}
The authors declare that they have no known competing financial interests or personal relationships that could have appeared to influence the work reported in this paper.

\section*{Acknowledgements}
This research was funded by the SUSTech Presidential Postdoctoral Fellowship, the China Postdoctoral Science Foundation (Grant No. 2025M773057), Shenzhen Science and Technology Program (Grant No. ZDSYS20210623092007023), and the Shenzhen Key Laboratory of Safety and Security for Next Generation of Industrial Internet, Southern University of Science and Technology.

\section*{Data availability}
The three real-world datasets are publicly available from their cited sources. ColoredMNIST is generated from MNIST using the protocol in Section~\ref{sec:experiments}. Our code will be publicly available upon acceptance.

\section*{Declaration of generative AI use}
The authors used generative AI tools only for language polishing and readability improvement.  All technical ideas, experiments, analyses, and conclusions were developed and verified by the authors, who take full responsibility for the content of this paper.

\appendix
\numberwithin{equation}{section}
\makeatletter
\@addtoreset{table}{section}
\@addtoreset{figure}{section}
\makeatother
\renewcommand{\thetable}{\thesection.\arabic{table}}
\renewcommand{\thefigure}{\thesection.\arabic{figure}}

\section{Proofs for Theoretical Analysis}\label{app:theory-proofs}

This appendix provides the complete proofs for Section~\ref{sec:theory}. The assumptions are local to the corresponding results and do not alter the algorithm in Section~\ref{sec:proposed-method}. Lemma~\ref{prop:T1} and Proposition~\ref{prop:Tmargin} are conditional structural results; Proposition~\ref{prop:T0} is deterministic on a fixed sample; and Proposition~\ref{prop:T2} is an algebraic bound for fixed cell risks. Corollary~\ref{cor:T2-oracle} adds a conditional population-risk transfer under total-variation alignment. No finite-sample empirical-to-population guarantee is claimed.

\begin{assumption}[Prototype geometry]\label{ass:proto}
The representation is sphere-normalized, $\|\phi(x)\|_2=1$. The labeled sample used to form the prototypes contains both classes, and each empirical class mean is non-zero:
\begin{equation*}
\bar\phi_c=\frac{1}{|\mathcal D_c|}\sum_{i:y_i=c}\phi(x_i),
\qquad
\mu_c=\bar\phi_c/\|\bar\phi_c\|_2.
\end{equation*}
The two empirical prototypes are distinct and the temperature satisfies $\tau>0$.
\end{assumption}

\subsection{Proof of Lemma~\ref{prop:T1}}\label{app:proof-T1}

\begin{proof}
Because $\phi(x)$ and both prototypes have unit norm, every cosine in Eqs.~\eqref{eq:proto-logit} and~\eqref{eq:proto-margin-score} is an inner product. Hence
\begin{equation*}
f_\phi(x)
=\tau^{-1}\langle\phi(x),\mu_1-\mu_0\rangle
=\tau^{-1}\langle\phi(x),w^\star\rangle.
\end{equation*}
If $y=1$, the margin is $\langle\phi(x),\mu_1-\mu_0\rangle$. If $y=0$, it is $\langle\phi(x),\mu_0-\mu_1\rangle$. Since $\widetilde y=+1$ in the first case and $-1$ in the second, both cases are summarized by Eq.~\eqref{eq:T1-equivalence}.
\end{proof}

\subsection{Proof of Proposition~\ref{prop:Tmargin}}\label{app:proof-Tmargin}

\begin{proof}
Lemma~\ref{prop:T1} and Eq.~\eqref{eq:shortcut-model} give
\begin{align*}
s(X;Y)
&=\langle\widetilde Y\phi(X),w^\star\rangle\\
&=\langle u+Qv+\xi,\alpha u+\beta v\rangle\\
&=\alpha\|u\|_2^2+Q\beta\|v\|_2^2
+\langle\xi,w^\star\rangle,
\end{align*}
where the last equality uses $u^\top v=0$. Taking the conditional expectation and using $\mathbb E[\langle\xi,w^\star\rangle\mid Q]=0$ yields
\begin{equation*}
\mathbb E[s\mid Q=q]
=\alpha\|u\|_2^2+q\beta\|v\|_2^2,
\qquad q\in\{-1,+1\}.
\end{equation*}
Subtracting the two conditional means proves Eq.~\eqref{eq:margin-gap}.

On the probability-one event where the bounded projected-noise condition holds, every conflicting example satisfies
\begin{equation*}
s\le \alpha\|u\|_2^2-\beta\|v\|_2^2+\delta,
\end{equation*}
whereas every aligned example satisfies
\begin{equation*}
s\ge \alpha\|u\|_2^2+\beta\|v\|_2^2-\delta.
\end{equation*}
The second lower bound exceeds the first upper bound by
$2(\beta\|v\|_2^2-\delta)>0$. Thus all conflicting margins precede all aligned margins in the empirical ordering almost surely. With distinct margins, the conventional empirical median assigns exactly $\lceil n/2\rceil$ observations to the low-margin cell. If $n_{\mathrm{cf}}\le n/2$, all conflicting observations lie in that cell, whose conflict fraction is $n_{\mathrm{cf}}/\lceil n/2\rceil$; for even $n$, this reduces to $2\pi_{\mathrm{cf}}$.
\end{proof}

\subsection{Proof of Proposition~\ref{prop:T0}}\label{app:proof-T0}

\begin{proof}
Write
\begin{equation*}
\bar\phi_c=|\mathcal D_c|^{-1}\sum_{i:y_i=c}\phi(x_i),
\qquad
\bar\phi'_c=|\mathcal D_c|^{-1}\sum_{i:y_i=c}\phi'(x_i).
\end{equation*}
The definition of $\varepsilon$ and the triangle inequality give
\begin{equation}
\|\bar\phi_c-\bar\phi'_c\|_2\le\varepsilon.
\label{eq:T0-mean}
\end{equation}
For non-zero vectors $a$ and $b$,
\begin{equation}
\left\|\frac{a}{\|a\|_2}-\frac{b}{\|b\|_2}\right\|_2
\le
\frac{2\|a-b\|_2}{\min(\|a\|_2,\|b\|_2)}.
\label{eq:T0-normalization}
\end{equation}
Applying Eq.~\eqref{eq:T0-normalization} to the two empirical class means and using Eq.~\eqref{eq:T0-mean} proves
\begin{equation}
\max_c\|\mu_c(\phi)-\mu_c(\phi')\|_2
\le 2\varepsilon/\beta_n.
\label{eq:T0-prototype}
\end{equation}

For any example and either class $c$, adding and subtracting
$\langle\phi'(x_i),\mu_c(\phi)\rangle$ gives
\begin{align*}
&|\langle\phi(x_i),\mu_c(\phi)\rangle
-\langle\phi'(x_i),\mu_c(\phi')\rangle|\\
&\qquad\le
\|\phi(x_i)-\phi'(x_i)\|_2
+\|\mu_c(\phi)-\mu_c(\phi')\|_2\\
&\qquad\le (1+2/\beta_n)\varepsilon.
\end{align*}
The margin is the difference of two such terms. Therefore
\begin{equation}
\max_i|s_i(\phi)-s_i(\phi')|
\le (2+4/\beta_n)\varepsilon=L_s\varepsilon.
\label{eq:T0-margin}
\end{equation}

Set $\eta=L_s\varepsilon$. A coordinatewise perturbation of a finite sequence by at most $\eta$ changes each order statistic by at most $\eta$. The same bound applies to the empirical median, whether defined as the central order statistic or the average of the two central order statistics. Equation~\eqref{eq:T0-margin} gives
\begin{equation}
|\hat m_n(\phi)-\hat m_n(\phi')|\le\eta.
\label{eq:T0-median}
\end{equation}
Equations~\eqref{eq:T0-prototype}--\eqref{eq:T0-median} prove Eq.~\eqref{eq:T0-stability}.

It remains to characterize which assignments can change. Consider an example assigned to the low-margin cell under $\phi$ and to the high-margin cell under $\phi'$. Equations~\eqref{eq:T0-margin} and~\eqref{eq:T0-median} imply
\begin{equation*}
\hat m_n(\phi)-2\eta < s_i(\phi)\le\hat m_n(\phi).
\end{equation*}
The reverse reassignment similarly implies
\begin{equation*}
\hat m_n(\phi)<s_i(\phi)\le\hat m_n(\phi)+2\eta.
\end{equation*}
Thus every reassigned example lies in
$\{|s_i(\phi)-\hat m_n(\phi)|\le2\eta\}$. Summing its indicator over the fixed sample and dividing by $n$ proves Eq.~\eqref{eq:T0-reassignment}.
\end{proof}

\subsection{Proof of Proposition~\ref{prop:T2} and Corollary~\ref{cor:T2-oracle}}\label{app:proof-T2}

\begin{proof}[Proof of Proposition~\ref{prop:T2}]
Let $d_e=r_e-\bar{\mathcal R}$, choose one $e^\star\in\arg\max_e d_e$, and set $M=d_{e^\star}$. Since $\sum_e d_e=0$, the $K-1$ deviations indexed by $e\ne e^\star$ sum to $-M$. By Cauchy--Schwarz,
\begin{equation*}
\sum_{e\ne e^\star}d_e^2
\ge
\frac{1}{K-1}
\left(\sum_{e\ne e^\star}d_e\right)^2
=
\frac{M^2}{K-1}.
\end{equation*}
Consequently,
\begin{equation*}
\mathcal P_{\mathrm{REx}}
=\frac{1}{K}\sum_e d_e^2
\ge
\frac{1}{K}\left(M^2+\frac{M^2}{K-1}\right)
=
\frac{M^2}{K-1}.
\end{equation*}
Therefore $M\le\sqrt{(K-1)\mathcal P_{\mathrm{REx}}}$, and adding $\bar{\mathcal R}$ proves Eq.~\eqref{eq:T2-bound}. For $K=2$, the two centered risks are $d$ and $-d$, so $\mathcal P_{\mathrm{REx}}=d^2$ and the bound is an equality.
\end{proof}

\begin{proof}[Proof of Corollary~\ref{cor:T2-oracle}]
Under the convention $\mathrm{TV}(P,Q)=\sup_A|P(A)-Q(A)|$, the variational characterization for a loss in $[0,B]$ gives
\begin{equation*}
|\mathbb E_{Q_g}\ell-\mathbb E_{Q_{e(g)}}\ell|
\le B\,\mathrm{TV}(Q_g,Q_{e(g)})
\le B\rho.
\end{equation*}
Hence $\mathcal R_g(f)\le\mathcal R_{e(g)}(f)+B\rho$. Maximizing over oracle groups gives
\begin{equation*}
\mathcal R_{\mathrm{wg}}^{(\mathcal G)}(f)
\le
\mathcal R_{\mathrm{wg}}^{(\hat{\mathcal E})}(f)+B\rho.
\end{equation*}
Applying Proposition~\ref{prop:T2} to the first term proves Eq.~\eqref{eq:T2-oracle}.
\end{proof}

\section{Supplementary Experiments and Tables}\label{app:tables}

\subsection{ColoredMNIST Protocol Comparison}
Table~\ref{tab:cmnist} separates two evaluation conventions. Published ERM and oracle-IRM results use reversed-environment test accuracy. The matched block uses WGA under the ProME protocol and supports the diagnostic discussion in Section~\ref{sec:diagnostic}.
\begin{table}[!t]
\centering
{
\caption{\textbf{ColoredMNIST under two evaluation protocols.} Published ERM and oracle-IRM results use reversed-environment test accuracy; the remaining rows use WGA under the ProME protocol. Results are mean $\pm$ standard deviation over ten seeds.}
\label{tab:cmnist}
\small
\setlength{\tabcolsep}{8pt}
\resizebox{\columnwidth}{!}{%
\begin{tabular}{lccc}
\toprule
Method  & Seeds & Metric & Score (\%) \\
\midrule
ERM (published)          & 10 & Test accuracy & 17.1\std{0.6} \\
IRM, oracle environments (published) & 10 & Test accuracy & 66.9\std{2.5} \\
\midrule
JTT                    & 10 & WGA & 68.14\std{2.09} \\
DFR                    & 10 & WGA & 70.86\std{0.55} \\
Loss-split (EIIL-style) & 10 & WGA & 74.09\std{5.75} \\
\textbf{\methodname}  & 10 & WGA & \textbf{75.77\std{5.48}} \\
\bottomrule
\end{tabular}}
}
\begin{flushleft}
\footnotesize \textit{Note.} Published ERM and IRM values are quoted from Table~2 of~\cite{arjovsky2019invariant}. The WGA block follows Section~\ref{sec:experiments}; its loss-split row changes the partition criterion within the same pipeline.
\end{flushleft}
\end{table}

\subsection{Single-Seed Pipeline Decomposition}

Table~\ref{tab:ablation} provides a seed-$42$ decomposition under a fixed budget and candidate set. Classifier repair and repair-aware selection produce the largest change in all displayed arms. Replacing the partition also changes the final result, while removing the invariance penalty has a smaller effect in this run. Because the table contains one seed, it is used as a mechanism diagnostic rather than a population estimate. The multi-seed controls in Table~\ref{tab:threearm} provide the main evidence.
\begin{table}
\centering
\caption{\textbf{Single-seed Waterbirds pipeline decomposition.} All variants use class-only prototypes and the same training budget, candidate set, and validation protocol. Pre-repair evaluates Stage~1; Final includes classifier repair and selection; $\textit{Gain}=\mathrm{Final}-\mathrm{Pre\mbox{-}repair}$.}
\label{tab:ablation}
\setlength{\tabcolsep}{2.5pt}
\resizebox{\columnwidth}{!}{%
\begin{tabular}{lcccccc}
\toprule
Variant & Partition & Invariance & DFR & Pre-repair & Final & \textit{Gain} \\
\midrule
\textbf{ProME (full)} & Prototype margin & \cmark & \cmark & 72.27 & \textbf{91.74} & \textbf{+19.47} \\
ERM-loss control & ERM loss & \cmark & \cmark & \textbf{85.36} & 86.29 & +0.93 \\
Without invariance & Prototype margin & \xmark & \cmark & 72.90 & 91.28 & +18.38 \\
\bottomrule
\end{tabular}
}
\end{table}

\subsection{Training Details}

\paragraph{Training schedules.}
The ERM warm-up lasts $5$ epochs on the real-world benchmarks and $100$ epochs on ColoredMNIST. Invariant training then runs for $3{,}000$ steps on Waterbirds and CivilComments, $5{,}000$ on CelebA, and $1{,}000$ on ColoredMNIST. AdamW uses learning rates of $10^{-4}$ for ResNet-50 and $2\times10^{-5}$ for BERT-base, with weight decay $10^{-4}$ and $10^{-2}$, respectively. Batch sizes are $32$, $64$, and $16$ on Waterbirds, CelebA, and CivilComments. The penalty weight increases linearly from $1$ to $3$ on the vision benchmarks and to $10$ on CivilComments. Prototypes are refreshed every $T_{\mathrm{proto}}=100$ steps by default, while the three-seed Waterbirds headline runs use $T_{\mathrm{proto}}=50$. The training loss is logistic cross-entropy. Loss boundedness is required only for the optional total-variation transfer in Corollary~\ref{cor:T2-oracle}. The controlled analyses in Sec.~\ref{sec:diagnostic} and Sec.~\ref{sec:post_repair} run repeated seeds and repeated repair passes over each candidate pool, so we run the seed-heavy sweeps on Waterbirds and ColoredMNIST and reserve CelebA, whose $162{,}770$ training images make each pass an order of magnitude more expensive, for the settings where it carries the comparison.

\bibliographystyle{elsarticle-num}
\bibliography{refs}

\begin{thebibliography}{10}
\expandafter\ifx\csname url\endcsname\relax
  \def\url#1{\texttt{#1}}\fi
\expandafter\ifx\csname urlprefix\endcsname\relax\def\urlprefix{URL }\fi
\expandafter\ifx\csname href\endcsname\relax
  \def\href#1#2{#2} \def\path#1{#1}\fi

\bibitem{sagawa2019distributionally}
S.~Sagawa*, P.~W. Koh*, T.~B. Hashimoto, P.~Liang,
  \href{https://openreview.net/forum?id=ryxGuJrFvS}{Distributionally robust
  neural networks}, in: International Conference on Learning Representations,
  2020.
\newline\urlprefix\url{https://openreview.net/forum?id=ryxGuJrFvS}

\bibitem{yang2023change}
Y.~Yang, H.~Zhang, D.~Katabi, M.~Ghassemi,
  \href{https://proceedings.mlr.press/v202/yang23s.html}{Change is hard: A
  closer look at subpopulation shift}, in: A.~Krause, E.~Brunskill, K.~Cho,
  B.~Engelhardt, S.~Sabato, J.~Scarlett (Eds.), Proceedings of the 40th
  International Conference on Machine Learning, Vol. 202 of Proceedings of
  Machine Learning Research, PMLR, 2023, pp. 39584--39622.
\newline\urlprefix\url{https://proceedings.mlr.press/v202/yang23s.html}

\bibitem{ye2026the}
W.~Ye, L.~Jiang, E.~Xie, G.~Zheng, Y.~Ma, X.~Cao, D.~Guo, D.~Qi, Z.~He,
  Y.~Tian, C.~W. Porter, M.~Coffee, Z.~Zeng, S.~Li, Z.~Wang, T.-H.~K. Huang,
  J.~M. Rehg, H.~Kautz, A.~Zhang,
  \href{https://openreview.net/forum?id=kIuqPmS1b1}{The clever hans mirage: A
  comprehensive survey on spurious correlations in machine learning},
  Transactions on Machine Learning Research (2026).
\newline\urlprefix\url{https://openreview.net/forum?id=kIuqPmS1b1}

\bibitem{geirhos2020shortcut}
R.~Geirhos, J.-H. Jacobsen, C.~Michaelis, R.~Zemel, W.~Brendel, M.~Bethge,
  F.~A. Wichmann, \href{https://doi.org/10.1038/s42256-020-00257-z}{Shortcut
  learning in deep neural networks}, Nature Machine Intelligence 2~(11) (2020)
  665--673.
\newblock \href {https://doi.org/10.1038/s42256-020-00257-z}
  {\path{doi:10.1038/s42256-020-00257-z}}.
\newline\urlprefix\url{https://doi.org/10.1038/s42256-020-00257-z}

\bibitem{liu2015deep}
Z.~Liu, P.~Luo, X.~Wang, X.~Tang, Deep learning face attributes in the wild,
  in: 2015 IEEE International Conference on Computer Vision (ICCV), 2015, pp.
  3730--3738.
\newblock \href {https://doi.org/10.1109/ICCV.2015.425}
  {\path{doi:10.1109/ICCV.2015.425}}.

\bibitem{liu2021just}
E.~Z. Liu, B.~Haghgoo, A.~S. Chen, A.~Raghunathan, P.~W. Koh, S.~Sagawa,
  P.~Liang, C.~Finn, \href{https://proceedings.mlr.press/v139/liu21f.html}{Just
  train twice: Improving group robustness without training group information},
  in: M.~Meila, T.~Zhang (Eds.), Proceedings of the 38th International
  Conference on Machine Learning, Vol. 139 of Proceedings of Machine Learning
  Research, PMLR, 2021, pp. 6781--6792.
\newline\urlprefix\url{https://proceedings.mlr.press/v139/liu21f.html}

\bibitem{pezeshki2024discovering}
M.~Pezeshki, D.~Bouchacourt, M.~Ibrahim, N.~Ballas, P.~Vincent, D.~Lopez-Paz,
  \href{https://proceedings.mlr.press/v235/pezeshki24a.html}{Discovering
  environments with {XRM}}, in: R.~Salakhutdinov, Z.~Kolter, K.~Heller,
  A.~Weller, N.~Oliver, J.~Scarlett, F.~Berkenkamp (Eds.), Proceedings of the
  41st International Conference on Machine Learning, Vol. 235 of Proceedings of
  Machine Learning Research, PMLR, 2024, pp. 40551--40569.
\newline\urlprefix\url{https://proceedings.mlr.press/v235/pezeshki24a.html}

\bibitem{sohoni2020no}
N.~Sohoni, J.~Dunnmon, G.~Angus, A.~Gu, C.~R\'{e},
  \href{https://proceedings.neurips.cc/paper_files/paper/2020/file/e0688d13958a19e087e123148555e4b4-Paper.pdf}{No
  subclass left behind: Fine-grained robustness in coarse-grained
  classification problems}, in: H.~Larochelle, M.~Ranzato, R.~Hadsell,
  M.~Balcan, H.~Lin (Eds.), Advances in Neural Information Processing Systems,
  Vol.~33, Curran Associates, Inc., 2020, pp. 19339--19352.
\newline\urlprefix\url{https://proceedings.neurips.cc/paper_files/paper/2020/file/e0688d13958a19e087e123148555e4b4-Paper.pdf}

\bibitem{creager2021environment}
E.~Creager, J.-H. Jacobsen, R.~Zemel,
  \href{https://proceedings.mlr.press/v139/creager21a.html}{Environment
  inference for invariant learning}, in: M.~Meila, T.~Zhang (Eds.), Proceedings
  of the 38th International Conference on Machine Learning, Vol. 139 of
  Proceedings of Machine Learning Research, PMLR, 2021, pp. 2189--2200.
\newline\urlprefix\url{https://proceedings.mlr.press/v139/creager21a.html}

\bibitem{han2024improving}
Y.~Han, D.~Zou, \href{https://proceedings.mlr.press/v235/han24g.html}{Improving
  group robustness on spurious correlation requires preciser group inference},
  in: R.~Salakhutdinov, Z.~Kolter, K.~Heller, A.~Weller, N.~Oliver,
  J.~Scarlett, F.~Berkenkamp (Eds.), Proceedings of the 41st International
  Conference on Machine Learning, Vol. 235 of Proceedings of Machine Learning
  Research, PMLR, 2024, pp. 17480--17504.
\newline\urlprefix\url{https://proceedings.mlr.press/v235/han24g.html}

\bibitem{zhang2022correct}
M.~Zhang, N.~S. Sohoni, H.~R. Zhang, C.~Finn, C.~Re,
  \href{https://proceedings.mlr.press/v162/zhang22z.html}{Correct-n-contrast: a
  contrastive approach for improving robustness to spurious correlations}, in:
  K.~Chaudhuri, S.~Jegelka, L.~Song, C.~Szepesvari, G.~Niu, S.~Sabato (Eds.),
  Proceedings of the 39th International Conference on Machine Learning, Vol.
  162 of Proceedings of Machine Learning Research, PMLR, 2022, pp.
  26484--26516.
\newline\urlprefix\url{https://proceedings.mlr.press/v162/zhang22z.html}

\bibitem{qiu2023simple}
S.~Qiu, A.~Potapczynski, P.~Izmailov, A.~G. Wilson,
  \href{https://proceedings.mlr.press/v202/qiu23c.html}{Simple and fast group
  robustness by automatic feature reweighting}, in: A.~Krause, E.~Brunskill,
  K.~Cho, B.~Engelhardt, S.~Sabato, J.~Scarlett (Eds.), Proceedings of the 40th
  International Conference on Machine Learning, Vol. 202 of Proceedings of
  Machine Learning Research, PMLR, 2023, pp. 28448--28467.
\newline\urlprefix\url{https://proceedings.mlr.press/v202/qiu23c.html}

\bibitem{liao2026decorr}
Y.~Liao, Q.~Wu, Y.~Wu, X.~Yan,
  \href{https://www.sciencedirect.com/science/article/pii/S0893608026001899}{Decorr:
  Environment partitioning for invariant learning and ood generalization},
  Neural Networks 199 (2026) 108727.
\newblock \href {https://doi.org/https://doi.org/10.1016/j.neunet.2026.108727}
  {\path{doi:https://doi.org/10.1016/j.neunet.2026.108727}}.
\newline\urlprefix\url{https://www.sciencedirect.com/science/article/pii/S0893608026001899}

\bibitem{wang2026environment}
S.~Wang, M.~Sun, Q.~Huang, Y.~Wang,
  \href{https://www.sciencedirect.com/science/article/pii/S0004370226000482}{Environment
  promoted invariant information learning for graph out-of-distribution
  generalization}, Artificial Intelligence 355 (2026) 104522.
\newblock \href {https://doi.org/https://doi.org/10.1016/j.artint.2026.104522}
  {\path{doi:https://doi.org/10.1016/j.artint.2026.104522}}.
\newline\urlprefix\url{https://www.sciencedirect.com/science/article/pii/S0004370226000482}

\bibitem{kirichenko2023last}
P.~Kirichenko, P.~Izmailov, A.~G. Wilson,
  \href{https://openreview.net/forum?id=THOOBy1uWVH}{Last layer re-training is
  sufficient for robustness to spurious correlations}, in: ICML 2022: Workshop
  on Spurious Correlations, Invariance and Stability, 2022.
\newline\urlprefix\url{https://openreview.net/forum?id=THOOBy1uWVH}

\bibitem{labonte2024towards}
T.~LaBonte, V.~Muthukumar, A.~Kumar,
  \href{https://proceedings.neurips.cc/paper_files/paper/2023/file/265bee74aee86df77e8e36d25e786ab5-Paper-Conference.pdf}{Towards
  last-layer retraining for group robustness with fewer annotations}, in:
  A.~Oh, T.~Naumann, A.~Globerson, K.~Saenko, M.~Hardt, S.~Levine (Eds.),
  Advances in Neural Information Processing Systems, Vol.~36, Curran
  Associates, Inc., 2023, pp. 11552--11579.
\newblock \href {https://doi.org/10.52202/075280-0509}
  {\path{doi:10.52202/075280-0509}}.
\newline\urlprefix\url{https://proceedings.neurips.cc/paper_files/paper/2023/file/265bee74aee86df77e8e36d25e786ab5-Paper-Conference.pdf}

\bibitem{hill2026on}
J.~C. Hill, T.~LaBonte, X.~Zhang, V.~Muthukumar,
  \href{https://openreview.net/forum?id=h81ztbrkFb}{On the unreasonable
  effectiveness of last-layer retraining}, Transactions on Machine Learning
  Research (2026).
\newline\urlprefix\url{https://openreview.net/forum?id=h81ztbrkFb}

\bibitem{koh2021wilds}
P.~W. Koh, S.~Sagawa, H.~Marklund, S.~M. Xie, M.~Zhang, A.~Balsubramani, W.~Hu,
  M.~Yasunaga, R.~L. Phillips, I.~Gao, T.~Lee, E.~David, I.~Stavness, W.~Guo,
  B.~Earnshaw, I.~Haque, S.~M. Beery, J.~Leskovec, A.~Kundaje, E.~Pierson,
  S.~Levine, C.~Finn, P.~Liang,
  \href{https://proceedings.mlr.press/v139/koh21a.html}{Wilds: A benchmark of
  in-the-wild distribution shifts}, in: M.~Meila, T.~Zhang (Eds.), Proceedings
  of the 38th International Conference on Machine Learning, Vol. 139 of
  Proceedings of Machine Learning Research, PMLR, 2021, pp. 5637--5664.
\newline\urlprefix\url{https://proceedings.mlr.press/v139/koh21a.html}

\bibitem{arjovsky2019invariant}
M.~Arjovsky, L.~Bottou, I.~Gulrajani, D.~Lopez-Paz,
  \href{https://doi.org/10.48550/arXiv.1907.02893}{Invariant risk
  minimization}, arXiv preprint arXiv:1907.02893 (2019).
\newblock \href {https://doi.org/10.48550/arXiv.1907.02893}
  {\path{doi:10.48550/arXiv.1907.02893}}.
\newline\urlprefix\url{https://doi.org/10.48550/arXiv.1907.02893}

\bibitem{idrissi2022simple}
B.~Y. Idrissi, M.~Arjovsky, M.~Pezeshki, D.~Lopez-Paz,
  \href{https://proceedings.mlr.press/v177/idrissi22a.html}{Simple data
  balancing achieves competitive worst-group-accuracy}, in: B.~Schölkopf,
  C.~Uhler, K.~Zhang (Eds.), Proceedings of the First Conference on Causal
  Learning and Reasoning, Vol. 177 of Proceedings of Machine Learning Research,
  PMLR, 2022, pp. 336--351.
\newline\urlprefix\url{https://proceedings.mlr.press/v177/idrissi22a.html}

\bibitem{nam2020learning}
J.~Nam, H.~Cha, S.~Ahn, J.~Lee, J.~Shin,
  \href{https://proceedings.neurips.cc/paper_files/paper/2020/file/eddc3427c5d77843c2253f1e799fe933-Paper.pdf}{Learning
  from failure: De-biasing classifier from biased classifier}, in:
  H.~Larochelle, M.~Ranzato, R.~Hadsell, M.~Balcan, H.~Lin (Eds.), Advances in
  Neural Information Processing Systems, Vol.~33, Curran Associates, Inc.,
  2020, pp. 20673--20684.
\newline\urlprefix\url{https://proceedings.neurips.cc/paper_files/paper/2020/file/eddc3427c5d77843c2253f1e799fe933-Paper.pdf}

\bibitem{li2024bias}
G.~Li, J.~Liu, W.~Hu, \href{https://openreview.net/forum?id=75OwvzZZBT}{Bias
  amplification enhances minority group performance}, Transactions on Machine
  Learning Research (2024).
\newline\urlprefix\url{https://openreview.net/forum?id=75OwvzZZBT}

\bibitem{zheng2024selfguided}
G.~Zheng, W.~Ye, A.~Zhang,
  \href{https://doi.org/10.24963/ijcai.2024/619}{Learning robust classifiers
  with self-guided spurious correlation mitigation}, in: Proceedings of the
  Thirty-Third International Joint Conference on Artificial Intelligence, IJCAI
  '24, 2024.
\newblock \href {https://doi.org/10.24963/ijcai.2024/619}
  {\path{doi:10.24963/ijcai.2024/619}}.
\newline\urlprefix\url{https://doi.org/10.24963/ijcai.2024/619}

\bibitem{han2024disagreement}
H.~Han, S.~Kim, H.~Joo, S.~Hong, J.~Lee,
  \href{https://proceedings.neurips.cc/paper_files/paper/2024/file/879c5890a9d2ecdcb590c9674cda4a59-Paper-Conference.pdf}{Mitigating
  spurious correlations via disagreement probability}, in: A.~Globerson,
  L.~Mackey, D.~Belgrave, A.~Fan, U.~Paquet, J.~Tomczak, C.~Zhang (Eds.),
  Advances in Neural Information Processing Systems, Vol.~37, Curran
  Associates, Inc., 2024, pp. 74363--74382.
\newblock \href {https://doi.org/10.52202/079017-2366}
  {\path{doi:10.52202/079017-2366}}.
\newline\urlprefix\url{https://proceedings.neurips.cc/paper_files/paper/2024/file/879c5890a9d2ecdcb590c9674cda4a59-Paper-Conference.pdf}

\bibitem{tsirigotis2024group}
C.~Tsirigotis, J.~Monteiro, P.~Rodriguez, D.~Vazquez, A.~Courville,
  \href{https://proceedings.neurips.cc/paper_files/paper/2023/file/b0d9ceb3d11d013e55da201d2a2c07b2-Paper-Conference.pdf}{Group
  robust classification without any group information}, in: A.~Oh, T.~Naumann,
  A.~Globerson, K.~Saenko, M.~Hardt, S.~Levine (Eds.), Advances in Neural
  Information Processing Systems, Vol.~36, Curran Associates, Inc., 2023, pp.
  56553--56575.
\newblock \href {https://doi.org/10.52202/075280-2468}
  {\path{doi:10.52202/075280-2468}}.
\newline\urlprefix\url{https://proceedings.neurips.cc/paper_files/paper/2023/file/b0d9ceb3d11d013e55da201d2a2c07b2-Paper-Conference.pdf}

\bibitem{le2025out}
P.~Q. Le, J.~Schl{\"o}tterer, C.~Seifert,
  \href{https://openreview.net/forum?id=EEeVYfXor5}{Out of spuriousity:
  Improving robustness to spurious correlations without group annotations},
  Transactions on Machine Learning Research (2025).
\newline\urlprefix\url{https://openreview.net/forum?id=EEeVYfXor5}

\bibitem{cohen2025identifying}
J.~P. Cohen, L.~Blankemeier, A.~S. Chaudhari,
  \href{https://openreview.net/forum?id=Utjw2z1ale}{Identifying spurious
  correlations using counterfactual alignment}, Transactions on Machine
  Learning Research (2025).
\newline\urlprefix\url{https://openreview.net/forum?id=Utjw2z1ale}

\bibitem{chen2024neuralcollapse}
Z.~Chen, M.~Zhang, S.~Cui, H.~Li, G.~Niu, M.~Gong, C.~Zhang, K.~Zhang,
  \href{https://proceedings.neurips.cc/paper_files/paper/2024/file/aa56c74513a5e35768a11f4e82dd7ffb-Paper-Conference.pdf}{Neural
  collapse inspired feature alignment for out-of-distribution generalization},
  in: A.~Globerson, L.~Mackey, D.~Belgrave, A.~Fan, U.~Paquet, J.~Tomczak,
  C.~Zhang (Eds.), Advances in Neural Information Processing Systems, Vol.~37,
  Curran Associates, Inc., 2024, pp. 93671--93689.
\newblock \href {https://doi.org/10.52202/079017-2970}
  {\path{doi:10.52202/079017-2970}}.
\newline\urlprefix\url{https://proceedings.neurips.cc/paper_files/paper/2024/file/aa56c74513a5e35768a11f4e82dd7ffb-Paper-Conference.pdf}

\bibitem{stromberg2024robustness}
N.~Stromberg, R.~Ayyagari, M.~Welfert, S.~Koyejo, R.~Nock, L.~Sankar,
  Robustness to subpopulation shift with domain label noise via regularized
  annotation of domains, Transactions on Machine Learning Research 2024,
  publisher Copyright: {\textcopyright} 2024, Transactions on Machine Learning
  Research. All rights reserved. (2024).

\bibitem{peters2016causal}
J.~Peters, P.~Bühlmann, N.~Meinshausen,
  \href{https://doi.org/10.1111/rssb.12167}{Causal inference by using invariant
  prediction: Identification and confidence intervals}, Journal of the Royal
  Statistical Society Series B: Statistical Methodology 78~(5) (2016)
  947--1012.
\newblock \href
  {http://arxiv.org/abs/https://academic.oup.com/jrsssb/article-pdf/78/5/947/49235444/jrsssb_78_5_947.pdf}
  {\path{arXiv:https://academic.oup.com/jrsssb/article-pdf/78/5/947/49235444/jrsssb_78_5_947.pdf}},
  \href {https://doi.org/10.1111/rssb.12167} {\path{doi:10.1111/rssb.12167}}.
\newline\urlprefix\url{https://doi.org/10.1111/rssb.12167}

\bibitem{rojas2018invariant}
M.~Rojas-Carulla, B.~Sch{{\"o}}lkopf, R.~Turner, J.~Peters,
  \href{http://jmlr.org/papers/v19/16-432.html}{Invariant models for causal
  transfer learning}, Journal of Machine Learning Research 19~(36) (2018)
  1--34.
\newline\urlprefix\url{http://jmlr.org/papers/v19/16-432.html}

\bibitem{ahuja2020invariant}
K.~Ahuja, K.~Shanmugam, K.~Varshney, A.~Dhurandhar,
  \href{https://proceedings.mlr.press/v119/ahuja20a.html}{Invariant risk
  minimization games}, in: H.~D. III, A.~Singh (Eds.), Proceedings of the 37th
  International Conference on Machine Learning, Vol. 119 of Proceedings of
  Machine Learning Research, PMLR, 2020, pp. 145--155.
\newline\urlprefix\url{https://proceedings.mlr.press/v119/ahuja20a.html}

\bibitem{liu2021kernelized}
J.~Liu, Z.~Hu, P.~Cui, B.~Li, Z.~Shen,
  \href{https://proceedings.neurips.cc/paper_files/paper/2021/file/b59a51a3c0bf9c5228fde841714f523a-Paper.pdf}{Integrated
  latent heterogeneity and invariance learning in kernel space}, in:
  M.~Ranzato, A.~Beygelzimer, Y.~Dauphin, P.~Liang, J.~W. Vaughan (Eds.),
  Advances in Neural Information Processing Systems, Vol.~34, Curran
  Associates, Inc., 2021, pp. 21720--21731.
\newline\urlprefix\url{https://proceedings.neurips.cc/paper_files/paper/2021/file/b59a51a3c0bf9c5228fde841714f523a-Paper.pdf}

\bibitem{lin2022zin}
Y.~Lin, S.~Zhu, L.~Tan, P.~Cui,
  \href{https://proceedings.neurips.cc/paper_files/paper/2022/file/9b77f07301b1ef1fe810aae96c12cb7b-Paper-Conference.pdf}{Zin:
  When and how to learn invariance without environment partition?}, in:
  S.~Koyejo, S.~Mohamed, A.~Agarwal, D.~Belgrave, K.~Cho, A.~Oh (Eds.),
  Advances in Neural Information Processing Systems, Vol.~35, Curran
  Associates, Inc., 2022, pp. 24529--24542.
\newblock \href {https://doi.org/10.52202/068431-1781}
  {\path{doi:10.52202/068431-1781}}.
\newline\urlprefix\url{https://proceedings.neurips.cc/paper_files/paper/2022/file/9b77f07301b1ef1fe810aae96c12cb7b-Paper-Conference.pdf}

\bibitem{rosenfeld2021risks}
E.~Rosenfeld, P.~Ravikumar, A.~Risteski,
  \href{https://par.nsf.gov/biblio/10222689}{The risks of invariant risk
  minimization}, International Conference on Learning Representations 9.
\newline\urlprefix\url{https://par.nsf.gov/biblio/10222689}

\bibitem{kamath2021does}
P.~Kamath, A.~Tangella, D.~Sutherland, N.~Srebro,
  \href{https://proceedings.mlr.press/v130/kamath21a.html}{Does invariant risk
  minimization capture invariance?}, in: A.~Banerjee, K.~Fukumizu (Eds.),
  Proceedings of The 24th International Conference on Artificial Intelligence
  and Statistics, Vol. 130 of Proceedings of Machine Learning Research, PMLR,
  2021, pp. 4069--4077.
\newline\urlprefix\url{https://proceedings.mlr.press/v130/kamath21a.html}

\bibitem{snell2017prototypical}
J.~Snell, K.~Swersky, R.~Zemel, Prototypical networks for few-shot learning,
  in: Proceedings of the 31st International Conference on Neural Information
  Processing Systems, NIPS'17, Curran Associates Inc., Red Hook, NY, USA, 2017,
  p. 4080–4090.

\bibitem{papyan2020prevalence}
V.~Papyan, X.~Y. Han, D.~L. Donoho,
  \href{https://www.pnas.org/doi/abs/10.1073/pnas.2015509117}{Prevalence of
  neural collapse during the terminal phase of deep learning training},
  Proceedings of the National Academy of Sciences 117~(40) (2020) 24652--24663.
\newblock \href
  {http://arxiv.org/abs/https://www.pnas.org/doi/pdf/10.1073/pnas.2015509117}
  {\path{arXiv:https://www.pnas.org/doi/pdf/10.1073/pnas.2015509117}}, \href
  {https://doi.org/10.1073/pnas.2015509117}
  {\path{doi:10.1073/pnas.2015509117}}.
\newline\urlprefix\url{https://www.pnas.org/doi/abs/10.1073/pnas.2015509117}

\bibitem{to2025diverse}
M.~N.~N. To, P.~F.~R. Wilson, V.~Nguyen, M.~Harmanani, M.~Cooper, F.~Fooladgar,
  P.~Abolmaesumi, P.~Mousavi, R.~G. Krishnan, Diverse prototypical ensembles
  improve robustness to subpopulation shift, in: Proceedings of the 42nd
  International Conference on Machine Learning, ICML'25, JMLR.org, 2025.

\bibitem{izmailov2022feature}
P.~Izmailov, P.~Kirichenko, N.~Gruver, A.~G. Wilson,
  \href{https://proceedings.neurips.cc/paper_files/paper/2022/file/fb64a552feda3d981dbe43527a80a07e-Paper-Conference.pdf}{On
  feature learning in the presence of spurious correlations}, in: S.~Koyejo,
  S.~Mohamed, A.~Agarwal, D.~Belgrave, K.~Cho, A.~Oh (Eds.), Advances in Neural
  Information Processing Systems, Vol.~35, Curran Associates, Inc., 2022, pp.
  38516--38532.
\newblock \href {https://doi.org/10.52202/068431-2791}
  {\path{doi:10.52202/068431-2791}}.
\newline\urlprefix\url{https://proceedings.neurips.cc/paper_files/paper/2022/file/fb64a552feda3d981dbe43527a80a07e-Paper-Conference.pdf}

\bibitem{nguyen2024group}
T.~LaBonte, J.~C. Hill, X.~Zhang, V.~Muthukumar, A.~Kumar,
  \href{https://proceedings.neurips.cc/paper_files/paper/2024/file/dc4db2ff2c1aefce3b594f821ea82fe6-Paper-Conference.pdf}{The
  group robustness is in the details: Revisiting finetuning under spurious
  correlations}, in: A.~Globerson, L.~Mackey, D.~Belgrave, A.~Fan, U.~Paquet,
  J.~Tomczak, C.~Zhang (Eds.), Advances in Neural Information Processing
  Systems, Vol.~37, Curran Associates, Inc., 2024, pp. 121598--121629.
\newblock \href {https://doi.org/10.52202/079017-3865}
  {\path{doi:10.52202/079017-3865}}.
\newline\urlprefix\url{https://proceedings.neurips.cc/paper_files/paper/2024/file/dc4db2ff2c1aefce3b594f821ea82fe6-Paper-Conference.pdf}

\bibitem{stromberg2024for}
N.~Stromberg, R.~Ayyagari, M.~Welfert, S.~Koyejo, R.~Nock, L.~Sankar,
  \href{https://openreview.net/forum?id=l8E68fD6yp}{For robust worst-group
  accuracy, ignore group annotations}, Transactions on Machine Learning
  Research (2024).
\newline\urlprefix\url{https://openreview.net/forum?id=l8E68fD6yp}

\bibitem{rudner2024mind}
T.~G.~J. Rudner, Y.~Shi~Zhang, A.~G. Wilson, J.~Kempe,
  \href{https://proceedings.mlr.press/v238/rudner24a.html}{Mind the {GAP}:
  Improving robustness to subpopulation shifts with group-aware priors}, in:
  S.~Dasgupta, S.~Mandt, Y.~Li (Eds.), Proceedings of The 27th International
  Conference on Artificial Intelligence and Statistics, Vol. 238 of Proceedings
  of Machine Learning Research, PMLR, 2024, pp. 127--135.
\newline\urlprefix\url{https://proceedings.mlr.press/v238/rudner24a.html}

\bibitem{stromberg2024enhancing}
N.~Stromberg, R.~Ayyagari, S.~Koyejo, R.~Nock, L.~Sankar,
  \href{https://proceedings.neurips.cc/paper_files/paper/2024/file/fba9133d7fa896ab9414ddd1a6b1ecbf-Paper-Conference.pdf}{Enhancing
  robustness of last layer two-stage fair model corrections}, in: A.~Globerson,
  L.~Mackey, D.~Belgrave, A.~Fan, U.~Paquet, J.~Tomczak, C.~Zhang (Eds.),
  Advances in Neural Information Processing Systems, Vol.~37, Curran
  Associates, Inc., 2024, pp. 139380--139401.
\newblock \href {https://doi.org/10.52202/079017-4424}
  {\path{doi:10.52202/079017-4424}}.
\newline\urlprefix\url{https://proceedings.neurips.cc/paper_files/paper/2024/file/fba9133d7fa896ab9414ddd1a6b1ecbf-Paper-Conference.pdf}

\bibitem{krueger2021out}
D.~Krueger, E.~Caballero, J.-H. Jacobsen, A.~Zhang, J.~Binas, D.~Zhang, R.~L.
  Priol, A.~Courville,
  \href{https://proceedings.mlr.press/v139/krueger21a.html}{Out-of-distribution
  generalization via risk extrapolation (rex)}, in: M.~Meila, T.~Zhang (Eds.),
  Proceedings of the 38th International Conference on Machine Learning, Vol.
  139 of Proceedings of Machine Learning Research, PMLR, 2021, pp. 5815--5826.
\newline\urlprefix\url{https://proceedings.mlr.press/v139/krueger21a.html}

\bibitem{borkan2019nuanced}
D.~Borkan, L.~Dixon, J.~Sorensen, N.~Thain, L.~Vasserman,
  \href{https://doi.org/10.1145/3308560.3317593}{Nuanced metrics for measuring
  unintended bias with real data for text classification}, in: Companion
  Proceedings of The 2019 World Wide Web Conference, WWW '19, Association for
  Computing Machinery, New York, NY, USA, 2019, p. 491–500.
\newblock \href {https://doi.org/10.1145/3308560.3317593}
  {\path{doi:10.1145/3308560.3317593}}.
\newline\urlprefix\url{https://doi.org/10.1145/3308560.3317593}

\bibitem{sagawa2020investigation}
S.~Sagawa, A.~Raghunathan, P.~W. Koh, P.~Liang,
  \href{https://proceedings.mlr.press/v119/sagawa20a.html}{An investigation of
  why overparameterization exacerbates spurious correlations}, in: H.~D. III,
  A.~Singh (Eds.), Proceedings of the 37th International Conference on Machine
  Learning, Vol. 119 of Proceedings of Machine Learning Research, PMLR, 2020,
  pp. 8346--8356.
\newline\urlprefix\url{https://proceedings.mlr.press/v119/sagawa20a.html}

\bibitem{nam2022spread}
J.~Nam, J.~Kim, J.~Lee, J.~Shin,
  \href{https://openreview.net/forum?id=_F9xpOrqyX9}{Spread spurious attribute:
  Improving worst-group accuracy with spurious attribute estimation}, in:
  International Conference on Learning Representations, 2022.
\newline\urlprefix\url{https://openreview.net/forum?id=_F9xpOrqyX9}

\bibitem{zhou2022model}
X.~Zhou, Y.~Lin, R.~Pi, W.~Zhang, R.~Xu, P.~Cui, T.~Zhang,
  \href{https://proceedings.mlr.press/v162/zhou22d.html}{Model agnostic sample
  reweighting for out-of-distribution learning}, in: K.~Chaudhuri, S.~Jegelka,
  L.~Song, C.~Szepesvari, G.~Niu, S.~Sabato (Eds.), Proceedings of the 39th
  International Conference on Machine Learning, Vol. 162 of Proceedings of
  Machine Learning Research, PMLR, 2022, pp. 27203--27221.
\newline\urlprefix\url{https://proceedings.mlr.press/v162/zhou22d.html}

\bibitem{qiao2025group}
R.~Qiao, Z.~Wu, J.~Wang, P.~W. Koh, B.~K.~H. Low,
  \href{https://proceedings.iclr.cc/paper_files/paper/2025/file/5b6c1a77153c9d11861e26f92780cf2f-Paper-Conference.pdf}{Group-robust
  sample reweighting for subpopulation shifts via influence functions}, in:
  Y.~Yue, A.~Garg, N.~Peng, F.~Sha, R.~Yu (Eds.), International Conference on
  Learning Representations, Vol. 2025, 2025, pp. 36959--36980.
\newline\urlprefix\url{https://proceedings.iclr.cc/paper_files/paper/2025/file/5b6c1a77153c9d11861e26f92780cf2f-Paper-Conference.pdf}

\end{thebibliography}

\end{document}